\documentclass[10pt,twocolumn,letterpaper]{article}

\usepackage[accsupp]{axessibility}
\usepackage[pagenumbers]{cvpr} 

\DeclareMathOperator{\atantwo}{atan2}

\usepackage{indentfirst}
\usepackage{adjustbox}
\usepackage{multirow}

\usepackage{amsmath}
\usepackage{amssymb}
\usepackage{mathtools}
\usepackage{amsthm}
\usepackage{nicefrac}

\theoremstyle{plain}
\newtheorem{theorem}{Theorem}[section]

\newtheorem{proposition}[theorem]{Proposition}

\newcommand{\STAB}[1]{\begin{tabular}{@{}c@{}}#1\end{tabular}}

\usepackage{algorithm}
\usepackage{algpseudocode}
\algrenewcommand\algorithmicrequire{\textbf{Input:}}
\algrenewcommand\algorithmicensure{\textbf{Output:}}

\definecolor{cvprblue}{rgb}{0.21,0.49,0.74}
\usepackage[pagebackref,breaklinks,colorlinks,allcolors=cvprblue]{hyperref}

\def\paperID{41315} 
\def\confName{CVPR}
\def\confYear{2026}

\title{Learning Coordinate-based Convolutional Kernels for\\Continuous SE(3) Equivariant and Efficient Point Cloud Analysis}

\author{
Jaein Kim$^{1}$,\quad\,Hee Bin Yoo$^{2}$\thanks{The majority of the work for this publication was done while these authors were in Seoul National University.}\;\,,\quad\,Dong-Sig Han$^{3}$\footnotemark[1]\;\,,\quad\,Byoung-Tak Zhang$^{1,4}$\\
$^{1}$Interdisciplinary Program in Neuroscience, Seoul National University \\
$^{2}$Département d’Informatique, École Normale Supérieure (ENS) \\
$^{3}$Department of Computing, Imperial College London \\
$^{4}$Department of Computer Science and Engineering, Seoul National University \\
{\tt\small \{jykim, btzhang\}@bi.snu.ac.kr}
}

\begin{document}
\maketitle
\begin{abstract}
A symmetry on rigid motion is one of the salient factors in efficient learning of 3D point cloud problems.
Group convolution has been a representative method to extract equivariant features, but its realizations have struggled to retain both rigorous symmetry and scalability simultaneously.
We advocate utilizing the intertwiner framework to resolve this trade-off,
but previous works on it, which did not achieve complete SE(3) symmetry or scalability to large-scale problems, necessitate a more advanced kernel architecture.
We present Equivariant Coordinate-based Kernel Convolution, or ECKConv.
It acquires SE(3) equivariance from the kernel domain defined in a double coset space, and its explicit kernel design using coordinate-based networks enhances its learning capability and memory efficiency.
The experiments on diverse point cloud tasks, e.g., classification, pose registration, part segmentation, and large-scale semantic segmentation, validate the rigid equivariance, memory scalability, and outstanding performance of ECKConv compared to state-of-the-art equivariant methods.
\end{abstract}

\section{Introduction}
Modern deep learning on point clouds has been dominating various tasks and applications in the 3D vision domain~\cite{guo2020deep,ioannidou2017deep,fei2024rotation}.
Yet, it has often exploited the prior that an input is aligned to the canonical pose, deteriorating its feasibility to the real-world problems~\cite{fei2024rotation}.
While training with data augmentation induces transformation robustness,
many studies~\cite{fei2024rotation,muller2021rotation,gerken2022equivariance,brehmer2025does} have shown that the reliance on augmentation is inefficient or underperforms compared to the model retaining symmetry in its architecture.
Thus, designing pose-equivariant model has been a significant topic in the 3D deep learning.
Group convolution is one of the representative approaches to grant such symmetry,
which derives the analogy from the translation equivariance of standard convolution towards the equivariance to group actions~\cite{cohen2016group}.

\begin{figure}[t]
    \centering
    \includegraphics[trim={0.8cm 0.3cm 0.3cm 0.3cm}, clip, width=0.94\columnwidth]{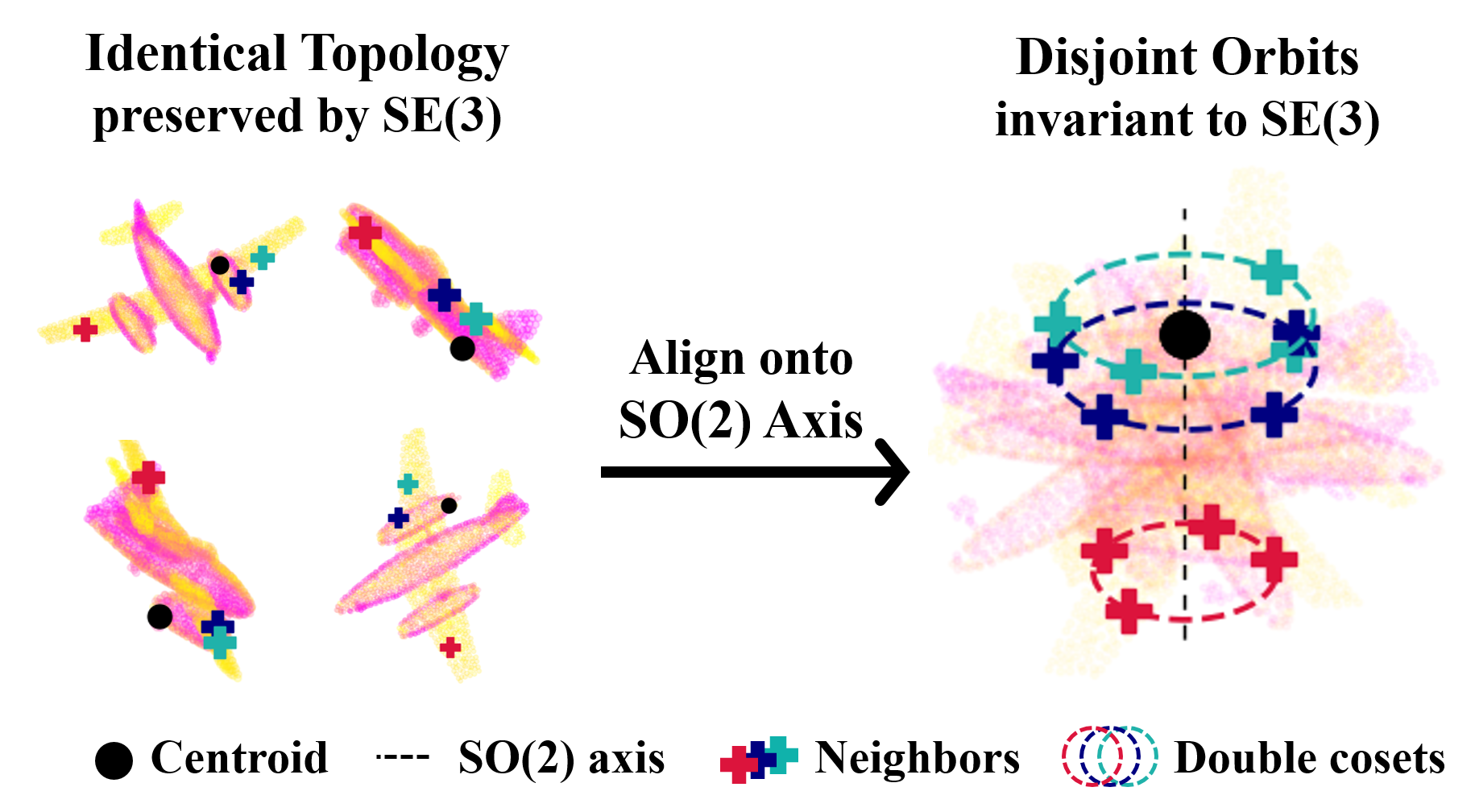}
    \caption{On the left, random SE(3) transformations are applied on the identical object point cloud. The black dots are centroids and colored plus signs designate neighbor points, where they are identical points with different poses. If these point clouds are transformed as the reference points are met on the subgroup, or the SO(2) axis, it is guaranteed that points with the identical topology lie on the disjoint orbit, or \textit{double coset}, around the axis. Therefore, the SE(3) equivariant operation is achievable by utilizing the unique parameter defining those orbits. We visualized this concept with an example drawn from the ModelNet40~\cite{wu20153d}.}
    \vspace{-1.0em}
    \label{fig:intro}
\end{figure}

Group convolution on the 3D point clouds has the trade-off between the strict equivariance and model scalability~\cite{chen2021equivariant,zhu2023e2pn}.
On the one hand, one can exploit a discrete group structure and directly expand kernel parameters to achieve scalability with restricted operations.
However, it causes a discrepancy between the model and the continuity of groups, \eg, discretizing 3D rotations from a compact SO(3) space~\cite{finzi2020generalizing}.
On the other hand, strict equivariance for continuous group spaces is guaranteed with the steerable convolution framework~\cite{cohen2016steerable},
which confines its features and kernels decomposable into the irreducible features and the direct sum of maps between them~\cite{thomas2018tensor,fuchs2020se}.
Such complex operations typically require expensive computational costs,
restraining its applicability to large-scale 3D tasks.

One of the extended framework of group convolution, the \textit{intertwiner} kernel convolution~\cite{cohen2018intertwiners,cohen2019general}, has been proposed to substitute the domain space from group to \textit{quotient} space, which is more plausible to represent the operation over the physical space.
Indeed, previous works based on this framework suggest more efficient discrete SE(3) equivariant convolution~\cite{zhu2023e2pn} or embed symmetry to continuous SO(3) actions~\cite{kim2024continuous} using \textit{implicit kernel}~\cite{romero2021ckconv,romero2021flexconv,zhdanov2024implicit}.
However, a continuous SE(3) equivariant method has not been presented within the \textit{intertwiner} framework, and their applications are not scalable enough to cover the large-scale 3D problems in the real-world.

This paper proposes Equivariant Coordinate-based Kernel Convolution (\textbf{ECKConv}), a continuous SE(3) equivariant convolutional network with scalability using \textit{intertwiner} kernel framework. 
We configure the domain of kernel to embed SE(3) symmetry and parse the symmetric parameters from the normalized local neighbors using coordinate-based networks.
With the \textit{explicit kernel} design, the parsed information outputs linear maps equivariant to complete SE(3) with the enhanced local features and the scalability in terms of memory efficiency.
Our method is verified through various translation/rotation invariant tasks and validates its symmetry on continuous rigid motions and practicality.

\noindent Our contributions are summarized as follows:
\begin{itemize}
    \item We introduce \textbf{ECKConv} which adopts the \textit{interwtiner} framework and computes the convolutional kernel from the parameters of a double coset space. It ensures a rigid symmetry against SE(3) from these parameters which corresponds to motion-invariant orbits as \Cref{fig:intro}.
    \item ECKConv extracts normalized double coset parameters from the local neighbors and computes the \textit{explicit kernel} maps using coordinate-based networks. It develops an enhanced learning of local geometries and the scalability of our method. We prove the efficiency of explicit kernel with respect to the backpropagation.
    \item We verify the strict SE(3) equivariance and memory scalability of ECKConv through the classification and pose registration tasks in ModelNet40~\cite{wu20153d}.
    The part segmentation in ShapeNet~\cite{chang2015shapenet} and the semantic segmentation in S3DIS~\cite{armeni20163d} also validate its advanced local feature learning and practicability to the large-scale problems.
\end{itemize}
\section{Related Works}
\subsection{Deep Learning for 3D Point Cloud}
PointNet~\cite{qi2017pointnet} and PointNet++~\cite{qi2017pointnet++} are early representative methods in point cloud deep learning.
They introduce a weight-sharing MLP, \ie, a one-dimensional convolution, to learn point features,
and PointNet++ especially utilizes local grouping and multi-scale (or resolution) mechanisms to enhance local geometry learning.
Despite such techniques, they are ineffective in aggregating local features, relying on averaging or pooling among grouped points.

Convolution methods in point cloud learning retain the inductive bias of locality in every layer while coping with irregularly distributed samples in the cloud.
Some early methods such as PointCNN~\cite{li2018pointcnn} and PointConv~\cite{wu2019pointconv} project distributed local features into regular features by learnable permutations or estimated density.
KPConv~\cite{thomas2019kpconv} defines its kernel for irregular grid as the combination of weight bases, whose coefficients are gated by the distance between input and anchor coordinates.
DGCNN~\cite{wang2019dynamic} constructs a local cluster at each point by K-NN in the latent feature space, which aggregates point features independently to the distance in the coordinate space.
It is also noteworthy that some studies adopt attention mechanisms and dynamically adjust the adjacency among point features~\cite{guo2021pct,zhao2021point,wu2022point,wu2024point}.
Yet, they are dependent on the data augmentations to induce symmetry to rigid motions only acting globally.


\subsection{Group Equivariant Point Cloud Convolution}
Group convolution is founded on either discrete group or steerable convolution,
rooted from GCNN~\cite{cohen2016group} and Steerable CNN~\cite{cohen2016steerable}.
TFN~\cite{thomas2018tensor}
decomposes the action of SO(3) and parametrizes the combination of equivariant bases to receive coordinate differences between points.
\citet{fuchs2020se} attach the attention mechanism among neighbors in TFN,
where query and key are also the direct sum of learnable equivariant bases.
As these steerable methods are not scalable to practical 3D vision problems, EPN~\cite{chen2021equivariant} and E2PN~\cite{zhu2023e2pn} suggest discrete group methods based on KPConv~\cite{thomas2019kpconv},
prescribing the assignment between discrete group or quotient space elements to replace the distance-based gate function.
They show a capacity comparable to non-equivariant methods,
but their symmetry is biased to the discretized rotations due to their attentive pooling over group dimensions.
Recent studies have aimed to address both limitations in the previous works.
CSEConv~\cite{kim2024continuous} suggests an implicit kernel invariant to continuous SO(3) by mapping double coset coordinates with learnable networks.
\citet{weijler2025efficient} propose a stochastic and efficient SE(3) convolution with lifted and sampled local coordinates.
However, the symmetry of CSEConv is restricted to SO(3), and their performance in object-level analysis is limited compared to existing point cloud methods.

\subsection{Equivariance outside Group Convolution}
It is possible to implement symmetric networks outside the group convolution framework to avoid its trade-off.
Some studies suggest model-agnostic frameworks that embed group equivariance into non-equivariant networks.
For instance, Vector Neuron~\cite{deng2021vector} suggests SO(3) equivariant framework that augments a feature space into a bundle of 3D vectors.
Frame Averaging~\cite{puny2022frame} proposes a theoretical framework that approximate integration over group into averaging over equivariant \textit{frame}, the finite subset of the group.
\citet{kaba2023equivariance} and \citet{mondal2023equivariant} introduce to canonicalize the input orientation via additional equivariant networks.
Although they retain the capacity of their foundation networks, they have limited symmetry to global SO(3) actions~\cite{deng2021vector}, or require high computational costs, pretrained networks~\cite{kaba2023equivariance}, or the prior to the input orientation~\cite{mondal2023equivariant}.

Meanwhile, there are studies that extract rotation invariant coordinate representations among point cloud clusters~\cite{zhang2019rotation,zhang2022riconv++,su2025ri}.
They have a similarity to ours in utilizing information invariant to rigid transformations during the operation.
However, their heuristic design and selection of such information make these methods susceptible to local references~\cite{zhang2019rotation}, utilize ambiguous representations to a certain situation~\cite{zhang2022riconv++}, or rely on the assumption that global transformation is equally applied to every local region~\cite{su2025ri}.
These limitations necessitate the study that addresses the trade-off by the group convolution within its field.

\section{Group Equivariant Convolution Framework using Intertwiner Kernel}
This section outlines the mathematical framework and the motivation of our architectural design.
We refer the readers to the previous works~\cite{cohen2018intertwiners,cohen2019general,zhu2023e2pn,kim2024continuous} for a better understanding of detailed derivations.
A group convolution is naturally derived from the standard convolution, simply substituting its symmetry against shifts into inverse actions closed on the group space~\cite{cohen2016group,cohen2019general}:
\begin{equation}\label{eq:basic_gconv}
(f \ast \kappa)(g)=\int_{G}\,\kappa(g^{-1}h)f(h)dh,
\end{equation}
where $f$ is an input feature, $\kappa$ is a kernel, and $g$ is a coordinate of point in the group space.
This expression is sufficient for modelling the equivariance to discrete groups.
However, real-world problems often require the symmetry to continuous groups,
and a naive integration on continuous group space is computationally intractable for data acquired from a physical space.

\citet{cohen2018intertwiners, cohen2019general} propose the framework that models the group equivariant convolution as a general linear map between fields by introducing the intertwiner $\text{Hom}(V_1, V_2)$ to the kernel function.
$V_1$ and $V_2$ are fields of the input and output feature spaces, whose coordinates are from a homogeneous space.
Notably, the convolution over group can be replaced into the equivalent operation over a quotient space, such as a 3D sphere, with $\kappa: G/H_1 \to \text{Hom}(V_1, V_2)$:
\begin{equation}\label{eq:quotient_conv}
\begin{gathered}
(f \ast \kappa)(x)=\\
\int_{G/H_1}\,\kappa(s_2(x)^{-1}y)\rho_1(\mathrm{h}_1(y, s_2(x)^{-1}))f(y)dy,
\end{gathered}
\end{equation}
where $\rho_1: G \to \text{GL}(V_1)$ is a representation, $s_2: G/H_2 \to G$ is a section map for group $G$ and subgroup $H_2$, and $\mathrm{h}_1(x, g)=s_1(gx)^{-1}gs_1(x)$ for $g \in G$, a section map $s_1$ for subgroup $H_1$, and $x \in G/H_{1}$.
For convenience, when $G=SO(3)$ and $H_2=SO(2)$, assume that $s_2(x)^{-1}$ transforms $x$ to the reference of $H_2$, \eg, the north pole.
The constraint on kernel $\kappa(\cdot)$ is required to retain the equivariance over group $G$ as \Cref{eq:constraint}, where the kernel preserves the symmetry to actions of subgroup $H_2$:
\begin{equation}\label{eq:constraint}
    \kappa(hy) = \rho_2(h)\kappa(y)\rho_1(\mathrm{h}_1(y,h)^{-1}),\;h \in H_2.
\end{equation}
\citet{cohen2018intertwiners, cohen2019general} also derive that $\kappa(\cdot)$ with \Cref{eq:constraint} has an equivalent function whose domain is in double coset space, where double cosets are disjoint bundles of quotient space elements transitive to $H_2$ actions.
To concretize $G$-equivariant operation, previous works~\cite{zhu2023e2pn,kim2024continuous} specify some configurations that $H_1\!=\!H_2\!=\!\text{SO(2)}$ and  $\rho_1\!=\!\rho_2\!=\!\text{id}$, assuming that $V_1$ and $V_2$ are invariant.
Consequently, the kernel function becomes an unconstrained function as:
\begin{equation}\label{eq:constraint_equivalent}
\begin{gathered}
(f \ast \kappa)(x)=\int_{G/H_1}\,\kappa(s_2(x)^{-1}y)f(y)dy,\\
\text{s.t.}\;\;\kappa: H_2 \backslash G / H_1 \to \text{Hom}(V_1, V_2),
\end{gathered}
\end{equation}
where the map from quotient space to double coset space is omitted with an abuse of notation.
Those previous methods, however, simplified the operation to the integration on the discretized $\text{SE(3)}/\text{SO(2)}$ quotient space~\cite{zhu2023e2pn} or confined symmetry to continuous SO(3)~\cite{kim2024continuous}.

\begin{figure*}[ht]
    \centering
    \begin{subfigure}[b]{0.35\textwidth}
    \centering
    \includegraphics[width=0.95\textwidth]{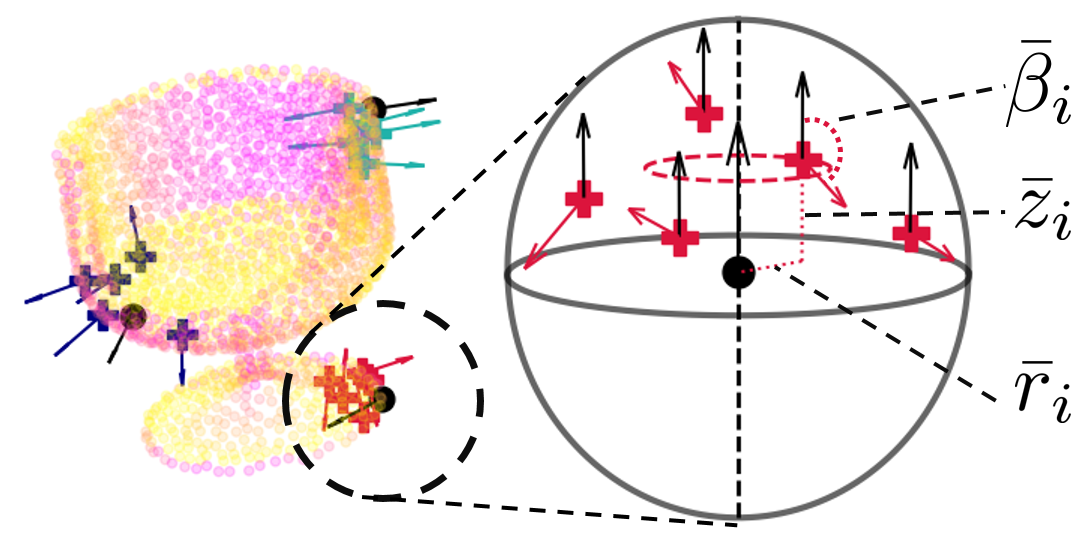}
    \caption{\label{fig:grouping_to_dc}}
    \end{subfigure}
    \hfil
    \begin{subfigure}[b]{0.6\textwidth}
    \centering
    \includegraphics[width=\textwidth]{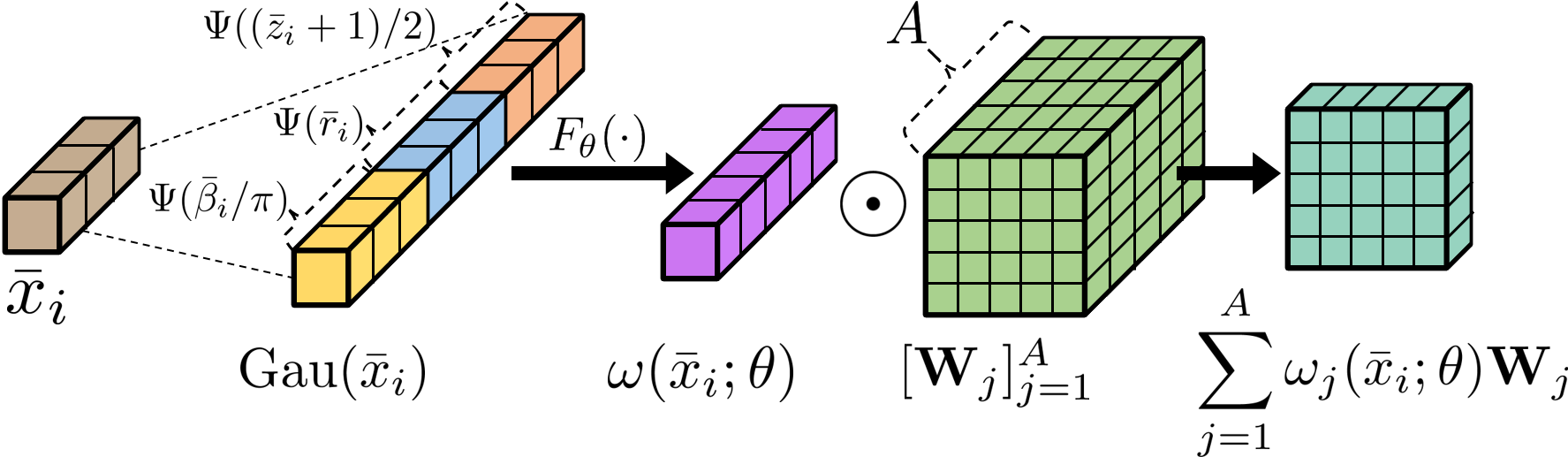}
    \caption{\label{fig:filter_arch}}
    \end{subfigure}
    \caption{The implementation details in ECKConv. (a) The neighbor points within a radius $r$ around each centroid are sub-sampled by ball query~\cite{qi2017pointnet++}. After they are aligned by the inverse of section map as \Cref{eq:cse2conv_op}, $\bar{\beta}_g$, $\bar{r}_g$, and $\bar{z}_g$, which correspond to the double coset parameters, are acquirable as depicted. These neighbor points are actually computed from the ModelNet40~\cite{wu20153d} object with normal vectors.
    (b) The computation of explicit kernel in ECKConv. First, it maps a double coset parameter $\bar{x}_i=[\bar{\beta}_i, \bar{r}_i, \bar{z}_i]$ into Gaussian embedding $[\Psi(\nicefrac{\bar{\beta}_i}{\pi}),\Psi(\bar{r}_i),\Psi(\nicefrac{(\bar{z}_i+1)}{2})]$~\cite{zheng2021rethinking,zheng2022trading}. Then the embedding is projected by $F_\theta$ to $\omega(\bar{x};\theta)$, a coefficient vector with the dimension $A$.
    The kernel value of $\kappa(\cdot)$ is gained from the weighted summation of learnable bases $[\mathbf{W}_{j}]^{A}_{j=1}$ by $\omega(\bar{x}_{i};\theta)$, \ie, $\sum^{A}_{j=1}\omega_{j}(\bar{x}_{i};\theta)\mathbf{W}_j$.
    }\label{fig:impl_change}
\end{figure*}
\section{Method}
This section delineates the methodology of ECKConv.
First, we briefly mention how the domain of kernel is defined for the SE(3) equivariance of convolution.
Then, we explain how the explicit kernel architecture computes our kernel value from double cosets using coordinate-based networks with the theoretical analysis on its scalability. 
Finally, we suggest the convolution block and its residual architecture for solving the point cloud problems.

\subsection{Continuous SE(3) Equivariant Convolution}\label{sec:se3_equiv_conv}
ECKConv expands the double coset space (\ref{eq:constraint_equivalent}) equipped with rotation and translation symmetry by configuring $G$ into $\text{SE(3)}=\text{SO(3)}\ltimes\mathbb{R}^3$.
We adopt the derivation by \citet{zhu2023e2pn}; an element $g$ in the double coset space is expressible as the pair of ZYZ Euler angles and translation when $G=\text{SE(3)}$ and $H_1\!=\!H_2\!=\!\text{SO(2)}$.
Let us denote that $g \equiv \big(R_z(\alpha_g)R_y(\beta_g)R_z(\gamma_g),\;[x_g,y_g,z_g]\big) \in G$, $h_1 = R_z(\gamma_{h_1}) \in H_1$, and $h_2 = R_z(\gamma_{h_2}) \in H_2$.
Then, we can write a double coset element $g$ as follows:
\begin{equation}\label{eq:dcp}
\begin{gathered}
    H_2gH_1 = \Big\{ \big(R_z(\alpha_g+\gamma_{h_2})R_y(\beta_g)R_z(\gamma_g+\gamma_{h_1}),\\ R_z(\gamma_{h_2}+\gamma_t)[r_g,0,z_g]^\top\big)\,\big|\,\forall\;h_1, h_2 \in \text{SO(2)}\Big\},
\end{gathered}
\end{equation}
where $\gamma_t\!=\!\atantwo(y_g,x_g)$ and $r_g\!=\!\sqrt{x^2_g+y^2_g}$ denote displacements.
These notations are useful since the whole set grouped by Z-axis rotations from $\text{SO(2)}$ is an element of double coset.
To summarize, a double coset element in $\text{SO(2)}\backslash\text{SE(3)}/\text{SO(2)}$ are uniquely representable as $[\beta_g, r_g, z_g]$,
and the convolution equivariant to continuous SE(3) can be built with the kernel depends on them.

\subsection{Implementations of ECKConv}\label{sec:scalable_impl}
We realize the framework in \Cref{sec:se3_equiv_conv} actually equivariant to continuous SE(3) actions unlike the previous work symmetric to discrete SE(3)~\cite{zhu2023e2pn}.
Let us suppose the locality assumption and the integration approximation rule from \citet{kim2024continuous}.
Then, the formulation of ECKConv is written as follows:
\begin{equation}\label{eq:cse2conv_op}
\begin{gathered}
    (f \ast \kappa)(x)=\sum_{x_i \in \mathcal{N}(x)}\kappa(s(x)^{-1}x_i)f(x_i),\\
    \text{s.t.}\;\;\kappa: \text{SE(3)}/\text{SO(2)} \to \mathbb{R}^{C_\text{out}\,\times\,C_\text{in}},\\
    f: \text{SE(3)}/\text{SO(2)} \to \mathbb{R}^{C_\text{in}\,\times\,1},
\end{gathered}
\end{equation}
where $s:SE(3)/SO(2) \to SE(3)$ is a section map.
The domain space of \Cref{eq:cse2conv_op} is factorizable into $\mathbb{R}^3 \times S^2$ since $\text{SE(3)}=\text{SO(3)}\ltimes\mathbb{R}^3$ and $\text{SO(3)}/\text{SO(2)}\equiv S^2$.
The coordinates of point clouds naturally represent $\mathbb{R}^3$.
The $S^2$ space is manifestible with any representations equivariant to SE(3) that map 3D coordinates to unit vectors. For instance, we can use either the surface normal vectors or heuristic augmentation as \Cref{alg:supp_coset_gen} in the supplementary material.

\subsubsection{Local Neighboring and Double Coset Encoding}
As we augment the domain into $\mathbb{R}^3 \times S^2$, let us assume a point cloud is composed of tuples of coordinate and normal vector $(\mathbf{x}, \mathbf{n})$.
Then, the local neighborhood, \ie, the region where we integrate, is determined by a ball query~\cite{qi2017pointnet++}, which is defined as $\mathcal{N}((\mathbf{x},\mathbf{n});r, k)=\{(\mathbf{x}_i, \mathbf{n}_i) \in \mathcal{P}\,|\,\|\mathbf{x} - \mathbf{x}_i\|_{2} \le r,\;i=1, ..., k\}$, where $r$ is a radius and $k$ is the maximal number of samples in ball query algorithm.
It is a well-established approach that helps utilize the geometry of input points in the Euclidean space and normalize the scale of receptive field by convolution,
suggested by many previous works with the locality assumption~\cite{qi2017pointnet++,thomas2019kpconv,chen2021equivariant,zhu2023e2pn}.

We encode the double coset parameters $[\bar{\beta}_i, \bar{r}_i, \bar{z}_i]$ of each neighbor point $x_i = (\mathbf{x}_i, \mathbf{n}_i)$ referring to the local coordinate system of ball query centroid $x = (\mathbf{x}, \mathbf{n})$.
These parameters correspond to $[{\beta}_g, {r}_g, {z}_g]$ of $s(x)^{-1}x_i$ from \Cref{eq:dcp,eq:cse2conv_op} and are acquired as follows:
\begin{equation}\label{eq:se3conv_dc_encode}
\begin{gathered}
    \bar{x}_i = [\bar{\beta}_i, \bar{r}_i, \bar{z}_i],\;\;\text{s.t.}\;\bar{\beta}_i = \arccos(\mathbf{n}^{\top}\!\cdot\mathbf{n}_i),\\
    \bar{z}_i = (\mathbf{n^{\top}}\cdot\dfrac{\Delta_i}{\|\Delta_i\|})\!\cdot\|\Delta_i\|/r,\\
    \bar{r}_i = \sqrt{1-(\mathbf{n^{\top}}\cdot\dfrac{\Delta_i}{\|\Delta_i\|})^2}\!\cdot\|\Delta_i\|/r,
\end{gathered}
\end{equation}
where $\Delta_i = \mathbf{x}_i - \mathbf{x}$ and $r$ is a ball query raidus.
The geometric interpretation of \Cref{eq:se3conv_dc_encode} is in \Cref{fig:grouping_to_dc},
where the local coordinate system by $x$ is aligned to a zero-centered Z-axis and SE(3) invariant information of $x_i$ is encoded to $\bar{x}_i$.
Thus, the operation only processes the double coset parameters in the scale-normalized region and enables each ECKConv layer to learn SE(3)-equivariant local geometries.

\subsubsection{Explicit Kernel using Coordinate-based Networks}
An explicit kernel function $\kappa(\cdot)$ consists of the combination of learnable linear maps as \Cref{eq:cse2conv_filter},
inspired by various convolution methods that compute the integration over irregular samples~\cite{thomas2019kpconv,chen2021equivariant,zhu2023e2pn,thomas2024kpconvx}.
\begin{equation}\label{eq:cse2conv_filter}
\begin{gathered}
    \kappa(s(x)^{-1}x_i\,;\mathbf{\theta},\big[\mathbf{W}_{j}\big]^{A}_{j=1}) = \sum^{A}_{j=1} \omega_j(\bar{x}_i\,;\mathbf{\theta})\,\mathbf{W}_j,\\
    \text{s.t.}\;s(x)^{-1}x_i\in\mathbb{R}^3 \times S^2,\;\bar{x}_i\in\text{SO(2)} \backslash \text{SE(3)} / \text{SO(2)},
\end{gathered}
\end{equation}
where input $s(x)^{-1}x_i$ from the quotient space is mapped to the double coset $\bar{x}_{i}$ by \Cref{eq:se3conv_dc_encode}, $\mathbf{W}_j \in \mathbb{R}^{C_\text{out} \times C_\text{in}}$, and $A$ is the number of anchor bases.
Previous studies determine the values of gate functions $\omega_j(\cdot)$ by the distance between input coordinates and anchor points in the 3D space.
However, they do not fit in our kernel function,
as it is obscure to measure the distance in the double coset space.

Instead, we incorporate the coordinate-based networks to surrogate the metric over the double coset space inspired by CSEConv~\cite{kim2024continuous}.
Neural networks $F_{\theta}(\cdot)$ map coordinates $\bar{x}_i$ into unbounded coefficients $\omega_j$ for each anchor basis as:
\begin{equation}\label{eq:cse2conv_rff}
\begin{gathered}
    \omega(\bar{x}_i\,;\mathbf{\theta}) = [\omega_j]^{A}_{j=1} = F_{\theta}\big(\text{Gau}(\bar{x}_i)\big) \in \mathbb{R}^A,\\
    \text{s.t.}\;\text{Gau}([\bar{\beta}_i,\bar{r}_i,\bar{z}_i])=[\Psi(\nicefrac{\bar{\beta}_i}{\pi}),\Psi(\bar{r}_i),\Psi(\nicefrac{(\bar{z}_i+1)}{2})],\\
    \Psi(x)=[\psi(x,0),\psi(x,\nicefrac{1}{d}),\cdots,\psi(x,\nicefrac{(d-1)}{d})],
\end{gathered}
\end{equation}
where $\psi(x, y)=\exp\big(-\nicefrac{(x-y)^2}{2\sigma^2}\big)$ is a Gaussian kernel with $\sigma=0.05$ and $\text{Gau}(\cdot)\in\mathbb{R}^{3d}$ is referred to as Gaussian embedding~\citep{zheng2021rethinking,zheng2022trading}.
By mapping $[\bar{\beta}_i, \bar{r}_i, \bar{z}_i]$ into a range $[0, 1]$, \citet{zheng2021rethinking,zheng2022trading} showed that Gaussian embedding provides the balance between memorization and generalization of coordinate learning with its bounded rank.
Then, combining \Cref{eq:cse2conv_op,eq:cse2conv_filter} is written as:
\begin{equation}\label{eq:cse2conv_original}
    (f \ast \kappa)(x)=\sum_{x_i \in \mathcal{N}(x)}\Big(\sum^{A}_{j=1} \omega_j(\bar{x}_i\,;\mathbf{\theta})\,\mathbf{W}_j\Big)f(x_i).
\end{equation}
This might be interpretable as an implicit kernel since $\omega(\cdot)$ is a coordinate-based networks.
However, we separate $\omega(\cdot)$ and basis maps $\mathbf{W}_j$ as an explicit kernel, and reorder \Cref{eq:cse2conv_original} as follows inspired by \citet{wu2019pointconv}:
\begin{equation}\label{eq:cse2conv_trick}
    (f \ast \kappa)(x)=\sum^{A}_{j=1}\;\mathbf{W}_{j}\!\!\!\sum_{x_i \in \mathcal{N}(x)}\!\omega_{j}(\bar{x}_i;\,\theta)f(x_i),
\end{equation}
which is the final formulation of ECKConv.

\subsection{Scalability Improvement in ECKConv}
Now we show the significance of this modification with respect to the computational benefit of \Cref{eq:cse2conv_trick}, although the operations in \Cref{eq:cse2conv_original,eq:cse2conv_trick} are equivalent~\cite{wu2019pointconv}, through the following proposition.
\begin{proposition}\label{propo}
    Let $C_\text{in}$ be the dimension of input feature, $C_\text{out}$ be the dimension of output feature, $K$ be the cardinality of neighbors, and $A$ be the cardinality of anchor bases.
    Then the cost from the derivative by $\theta$, which is the parameter of $\omega$, reduces from $\mathcal{O}(A\,K\,C_\text{in}\,C_\text{out})$ to $\mathcal{O}(A(K\,C_\text{in} + C_\text{in}\,C_\text{out}))$.
\end{proposition}
\begin{proof}
    Let us denote $Y=(f \ast \kappa)(x) \in \mathbb{R}^{C_\text{out}}$.
    Since the modification does not influence $\omega(x)$, the computation cost of $\nabla_{\theta}\omega$ is constant.
    Then one can distinguish the formulas of $\nabla_{\omega_{j}}Y$ according to the operations as follows: 
    \begin{equation}
        \nabla_{\omega_{j}}Y =
        \begin{cases}
            \sum^{K}_{i=1}\mathbf{W}_jf(x_i) & \text{from Equation~(\ref{eq:cse2conv_original}),}\\
            & \\
            \mathbf{W}_j\sum^{K}_{i=1}f(x_i) & \text{from Equation~(\ref{eq:cse2conv_trick}).}
        \end{cases}
    \end{equation}
    The cost by the summation of products between $\mathbf{W}_j$ and $f(x_i)$ is $\mathcal{O}(K\,C_\text{in}\,C_\text{out})$,
    and the cost by the product between $\mathbf{W}_j$ and the summation of $f(x_i)$s is $\mathcal{O}(K\,C_\text{in} + C_\text{in}\,C_\text{out})$.
    Since $\nabla_{\theta}Y=\big[\nabla_{\omega_j}Y\big]^{A}_{j=1}\nabla_{\theta}\omega$,
    the cost by $\nabla_{\theta}Y$ alters from $\mathcal{O}(A\,K\,C_\text{in}\,C_\text{out})$ to $\mathcal{O}(A(K\,C_\text{in} + C_\text{in}\,C_\text{out}))$.
\end{proof}
\noindent Proposition~\ref{propo} states that \Cref{eq:cse2conv_trick} is equivalent to \Cref{eq:cse2conv_original} with reduced memory cost during the backpropagation.
Considering the equivalence of \Cref{eq:cse2conv_original} to the implicit kernel, such as CSEConv~\cite{kim2024continuous}, it necessitates the explicit kernel for the scalability, enabling more parameters on each layer and deeper network architecture.

\subsection{Architecture of ECKConv Block and Residuals}
\begin{figure}[t]
    \centering
    \includegraphics[width=\columnwidth]{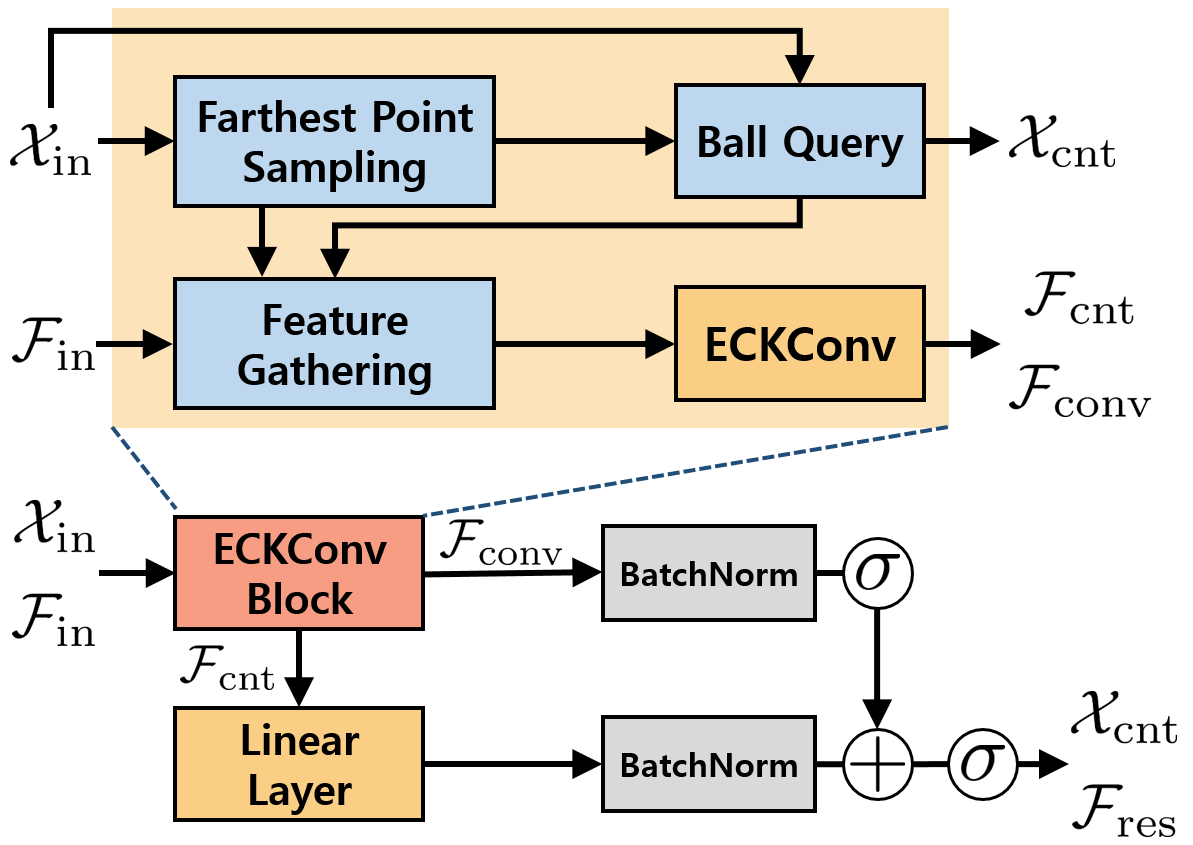}
    \caption{
    Architectures of ECKConv and its residual connection. We apply batch normalizations and activations ($\sigma$) in the order shown in the figure. As we only visualize abstract flow of the ECKConv block and its residual block, please refer to \Cref{sec:supp_layer_detail} in the supplementary material for more details.
    }
    \label{fig:layer_arch}
\end{figure}

The preceding components coalesce into a single block of ECKConv as \Cref{fig:layer_arch}.
We denote the input with $N$ points as $(\mathcal{X}_\text{in},\mathcal{F}_\text{in})$ where $\mathcal{X}_\text{in}=\{x^{n}|x^{n}=(\mathbf{x}^{n}, \mathbf{n}^{n})\}^{N}_{n=1}$, $\mathcal{F}_\text{in}=\{f^{n}|f^{n}=f(x^{n})\}^{N}_{n=1}$, and $\mathbf{x}^{n}$, $\mathbf{n}^{n}$, and $f^{n}$ are the coordinate, normal vector, and feature of $n$th point.
A farthest point sampling~\cite{qi2017pointnet++} chooses centroid points $\mathcal{X}_\text{cnt}$ and the block gathers features $\mathcal{F}_\text{cnt}$ that correspond to these centroids.
After a ball query finds neighbor points $\mathcal{N}(x^{n})$ for each centroid $x^{n}$ and their corresponding features are also gathered,
the ECKConv operates the convolution in \Cref{eq:cse2conv_trick} and outputs $\mathcal{F}_\text{conv}$.
The pair of $(\mathcal{X}_\text{cnt},\mathcal{F}_\text{conv})$ becomes an output point cloud of the ECKConv block.

Furthermore, we implement the residual connection block \cite{he2016deep} of ECKConv for the enhanced model capability as \Cref{fig:layer_arch}, inspired by previous point convolution methods~\cite{thomas2019kpconv,zhu2023e2pn}.
The features of centroids $\mathcal{F}_\text{cnt}$ are propagated by a linear layer, summed with $\mathcal{F}_\text{conv}$, and become new output features $\mathcal{F}_\text{res}$ paired with $\mathcal{X}_\text{cnt}$.

\section{Experiments}
We conducted classification and pose registration task in the ModelNet40~\cite{wu20153d} to verify the memory scalability, enhanced model capability, and rigorous SE(3) equivariance of ECKConv.
Then it is validated on more complex tasks, \ie, object part and indoor semantic segmentation, to verify its understanding of local geometries with equivariance and scalability to the real world.
Please refer to the supplementary material for detailed training configurations.

\noindent\textbf{Classification:}
It is the most standard task where a model classifies an input point cloud to one of 40 object categories.
As in previous studies~\cite{zhu2023e2pn,kim2024continuous}, we trained and evaluated models with or without applying SO(3) rotations, meaning 4 cases were reported for each model.

\noindent\textbf{Pose Registration:}
In this task, a model predicts the relative SE(3) transformation between two identical objects with different poses.
We measured the relative angle error and translation RMSE between the estimated and ground truth transformations for assessment.

\noindent\textbf{Object Part Segmentation:}
ShapeNet~\cite{chang2015shapenet} contains objects over 16 categories,
and point samples of the object are annotated with 50 part segmentation labels defined disjointly by their object category.
A model outputs part segmentation labels classifiable into 50 categories per every sample point, given the whole point cloud and its object label.
The segmentation performances were measured with mIoUs averaged over instances or part classes.

\noindent\textbf{Indoor Semantic Segmentation:}
We conducted a semantic segmentation using the S3DIS~\cite{armeni20163d} that contains indoor scenes sampled from real-world environments.
This benchmark is challenging due to its enormous scale, where point clouds are densely sampled from 271 rooms across 6 different areas covering over 6000$\;\!\!m^2$ and labeled with one of 13 semantic classes.
We used Area 5 as the test set and trained with the rest of Areas following previous works~\cite{thomas2019kpconv,thomas2024kpconvx,su2025ri}.

\subsection{Classification and Scalability Evaluation}
\begin{table}[t]
    \centering
    \caption{Classification accuracy on the ModelNet40 dataset. Since PRLC~\cite{mondal2023equivariant} requires the prior that input has an aligned orientation, it is impossible to apply SO(3) augmentation during its training. \textdagger~denotes the results reported from previous works~\cite{zhu2023e2pn,kim2024continuous}.}\label{tab:modelnet_cls}
    \resizebox{\columnwidth}{!}{{\large
    \begin{tabular}{l|cccc}
    \toprule[1.0pt]
    \multirow{2}{*}{Models} & \multicolumn{4}{c}{Classification Accuracy (\%)}\\
     & $I/I$ & $I/\text{SO(3)}$ & $\text{SO(3)}/I$ & $\text{SO(3)}/\text{SO(3)}$ \\
    \hline
    \hline
    &&&&\\[-0.8em]
    PointNet++~\cite{qi2017pointnet++}\textsuperscript{\textdagger}& 89.13 & 9.58  & 81.16 & 80.29 \\
    PointConv~\cite{wu2019pointconv}                               & 92.31 & 5.29  & 88.14 & 87.74 \\
    KPConv~\cite{thomas2019kpconv}\textsuperscript{\textdagger}    & 91.25 & 14.56 & 84.84 & 83.39 \\
    DGCNN~\cite{wang2019dynamic}                                   & 91.41 & 14.34 & 88.24 & 87.97 \\
    PCT~\cite{guo2021pct}\textsuperscript{\textdagger}             & 91.25 & 15.84 & 86.39 & 84.20 \\
    PTv3~\cite{wu2024point}                                        & \textbf{92.54} & 13.65 & 86.91 & 88.17 \\
    &&&&\\[-1.0em]
    \hline
    &&&&\\[-1.0em]
    VN~\cite{deng2021vector}                                       & 90.52 & 90.52 & 90.40 & 90.28 \\
    FA~\cite{puny2022frame}\textsuperscript{\textdagger}           & 82.25 & 82.25 & 82.01 & 82.01 \\
    LC~\cite{kaba2023equivariance}                                 & 88.33 & 88.87 & 88.09 & 88.12 \\
    PRLC~\cite{mondal2023equivariant}                              & 88.49 & 88.45 & -     & - \\
    &&&&\\[-1.0em]
    \hline
    &&&&\\[-1.0em]
    TFN~\cite{thomas2018tensor}\textsuperscript{\textdagger}       & 62.28 & 62.28 & 62.64 & 62.64 \\
    SE(3)-T~\cite{fuchs2020se}\textsuperscript{\textdagger}        & 71.60 & 71.60 & 73.01 & 73.01 \\
    EPN~\cite{chen2021equivariant}\textsuperscript{\textdagger}    & 91.45 & 31.08 & 86.60 & 86.60 \\
    E2PN~\cite{zhu2023e2pn}\textsuperscript{\textdagger}           & 91.58 & 44.47 & 89.47 & 88.58 \\
    CSEConv~\cite{kim2024continuous}\textsuperscript{\textdagger}  & 83.79 & 83.75 & 83.83 & 83.75 \\
    ECKConv-mini                                                   & 87.40 & 87.40 & 87.32 & 87.36 \\

    ECKConv                                                        & 90.52 & 90.52 & 90.19 & 90.19 \\
    ECKConv-Normal                                                 & 91.37 & \textbf{91.37} & \textbf{90.92} & \textbf{90.92} \\
    \bottomrule[1.0pt]
    \end{tabular}}}
\end{table}

\begin{table}[t]
    \centering
    \caption{Cost table of the equivariant classification baselines.}\label{tab:modelnet_cls_cost}
    \resizebox{\columnwidth}{!}{{\large
    \begin{tabular}{l|cc|cc|c}
    \toprule[1.0pt]
    \multirow{2}{*}{\begin{tabular}[c]{@{}l@{}}Models\\(Batch Size = 12)\end{tabular}} & \multicolumn{2}{c|}{Memory (GB) $\downarrow$} & \multicolumn{2}{c|}{$\#\text{Batch}/\text{sec}$ $\uparrow$} & \multirow{2}{*}{\#params}\\
    & Train & Eval & Train & Eval & \\
    \hline
    &&&&&\\[-1.0em]
    VN~\cite{deng2021vector} & 6.10  & \underline{1.98} & 4.34 & 8.70 & 2.9M\\
    EPN~\cite{chen2021equivariant} & 13.40 & 6.34 & 2.11 & 3.36 & 3.2M\\
    E2PN~\cite{zhu2023e2pn} & \textbf{3.94}  & 2.41 & \textbf{8.94} & \textbf{18.97} & 2.7M\\
    ECKConv & \underline{4.89}  & \textbf{1.28} & \underline{7.74} & \underline{15.09} & 27.7M\\ 
    &&&&&\\[-1.1em]
    \hline
    \hline
    &&&&&\\[-1.0em]
    CSEConv~\cite{kim2024continuous} & 2.95  & 0.44 & \textbf{\textit{30.71}} & \textbf{\textit{59.77}} & 2.1M\\
    ECKConv-mini & \textbf{\textit{0.72}} & \textbf{\textit{0.28}} & 23.13 & 42.74 & 1.9M\\ 
    \bottomrule[1.0pt]
    \end{tabular}}}
\end{table}

We prepared two variations of our classification models: they either used the true normal vector or \Cref{alg:supp_coset_gen} as the $S^2$ input (\textbf{ECKConv-Normal} and \textbf{ECKConv}).
Baselines were categorized into three types: the non-equivariant methods~\cite{qi2017pointnet++,wu2019pointconv,thomas2019kpconv,wang2019dynamic,guo2021pct,wu2024point}, the model-agnostic methods~\cite{deng2021vector,puny2022frame,kaba2023equivariance,mondal2023equivariant} based on DGCNN~\cite{wang2019dynamic},
and the group convolution methods~\cite{thomas2018tensor,fuchs2020se,chen2021equivariant,zhu2023e2pn,kim2024continuous} including both discrete and steerable methods.
Furthermore, we experimented \textbf{ECKConv-mini} which has identical model setups, \eg, the number of layers, layer-wise input/output dimensions, and the size of neighbors, to CSEConv~\cite{kim2024continuous} for the fair scalability comparison and substantiation of Proposition~\ref{propo}.

\Cref{tab:modelnet_cls} shows that ECKConv-Normal reached the best accuracy over 91\% when either the training or test set were rotated.
Even without the normal vector input, ECKConv reached comparable accuracy to Vector Neurons~\cite{deng2021vector} which is only equivariant to SO(3), and ECKConv-mini achieved nearly 4\%p higher accuracy than CSEConv.
These results validate the performance benefit of ECKConv against relying on augmentation and previous equivariant methods.

We measured the efficiency of equivariant baselines and ECKConv variants in \Cref{tab:modelnet_cls_cost}.
While ECKConv maintained an efficiency comparable to E2PN,
the comparison between CSEConv and ECKConv-mini revealed the benefit of explicit kernel in terms of memory as stated in Proposition~\ref{propo}.
In conclusion, ECKConv demonstrates its scalability regarding both model capacity and memory efficiency.

\subsection{Pose Registration}
\begin{table}[t]
    \centering
    \caption{Pose registration results on the ModelNet40. We reproduced the methods by \citet{chen2021equivariant} and \citet{zhu2023e2pn} for KPConv, EPN, and E2PN, which do not estimate the translation difference. Models were evaluated 10 times with different seeds.}\label{tab:modelnet_reg}
    \resizebox{\columnwidth}{!}{{\Large
    \begin{tabular}{l|c|c|c|c}
    \toprule[1.0pt]
    Models & Mean ($^\circ$) & Median ($^\circ$) & Max ($^\circ$) & tRMSE\\
    \hline
    \hline
    &&&&\\[-0.8em]
    KPConv~\cite{thomas2019kpconv} & $128.57 \pm 0.97$ & $123.22 \pm 0.52$ & $180.00 \pm 0.00$  & -\\
    DGCNN\,+\,DCP-v1~\cite{wang2019deep}     & $121.29 \pm 0.90$ & $137.19 \pm 1.95$ & $179.60 \pm 0.00$  & 0.081\\
    DGCNN\,+\,DCP-v2~\cite{wang2019deep}     & $120.07 \pm 1.28$ & $132.54 \pm 1.94$ & $179.60 \pm 0.00$  & 0.072\\
    \qquad\,w/ SO(3) Aug & $121.61 \pm 1.58$ & $137.62 \pm 1.85$ & $179.60 \pm 0.00$  & 0.075\\
    &&&&\\[-1.0em]
    \hline
    &&&&\\[-1.0em]
    EPN~\cite{chen2021equivariant} & $1.62   \pm 0.02$ & $2.40   \pm 0.08$ & $110.14 \pm 45.53$ & -\\
    E2PN~\cite{zhu2023e2pn}        & $11.70  \pm 0.25$ & $13.71  \pm 0.19$ & $163.62 \pm 15.70$ & -\\
    CSEConv~\cite{kim2024continuous}\,+\,DCP-v2 & $9.34  \pm 0.16$ & $2.93  \pm 0.07$ & $178.95 \pm 1.06$ & 0.17\\
    ECKConv\,+\,DCP-v2 & $\mathbf{0.63 \pm 0.00}$ & $\mathbf{0.53 \pm 0.00}$ & $\mathbf{8.57 \pm 2.40}$ & \textbf{0.010}\\
    \bottomrule[1.0pt]
    \end{tabular}}}
\end{table}
Our pose registration model consisted of ECKConv local feature extractor and DCP method~\cite{wang2019deep} to compute the relative pose between objects.
Models experimented by \citet{wang2019deep} were compared as baselines since they utilize DGCNN~\cite{wang2019dynamic} as a feature extractor.
We also included models from \citet{zhu2023e2pn} and the CSEConv model sharing the identical architecture to ours as the equivariant baselines.
Each metric represents the mean, median, maximum of error angles by rotations and RMSE of translation error.
The training was conducted with randomly sampled poses from $SE(3)$ and the entire object categories.

Table~\ref{tab:modelnet_reg} demonstrates the definite performance gap between our model and others in every metric.
Baselines with DGCNN backbones could not handle the rotation difference sampled from the complete SO(3) space.
We also verified that applying SO(3) augmentation on the source object during the training of DGCNN\,+\,DCP-v2 was unavailing either.
Other baselines, only structurally equivariant to discrete group or SO(3), were degraded by high maximum angle errors.
ECKConv with the closest point approach fit best to the pose registration task thanks to its embedded and thorough symmetry to any SE(3) actions.

\subsection{Object Part Segmentation}

Our part segmentation model adopted U-Net~\cite{ronneberger2015u} inspired approach by \citet{qi2017pointnet++} and adopted the feature interpolation module to up-sample the features of K-closest points.
Please refer to \Cref{sec:supp_pfp} in the supplementary material for the interpolation details.
We prepared Pointent++~\cite{qi2017pointnet++}, original DGCNN, model-agnostic methods based on DGCNN, and CSEConv baseline implemented with similar setups to ECKConv as baselines.

\Cref{tab:shapenetpart} shows that ECKConv achieved the instance averaged mIoUs over 83\%, the best among the baselines when the SO(3) is applied during the evaluation.
The part class averaged mIoUs (mcIoUs) of ECKConv also reached comparable scores to the augmented and model-agnostic baselines.
A qualitative analysis in \Cref{fig:part_seg} revealed that CSEConv failed to discern local geometries even in the simple objects.
Contrarily, ECKConv distinguished fine local features, \eg, the saddle and wheel of the bike, more accurately than other baselines.
Unlike the previous experiments where learning the global shape was enough to solve the task,
the performance gap in the segmentation task validates the outstanding capability of learning local geometries by ECKConv among the intertwiner convolution framework.

\begin{table}[t]
    \centering
    \caption{Part segmentation on the ShapeNet. mIoU denotes the IoU per object averaged among every instance, and mcIoU denotes that IoUs are averaged per part class before the total mean.}\label{tab:shapenetpart}
    \resizebox{\columnwidth}{!}{{\large
    \begin{tabular}{l|cc|cc|cc}
    \toprule[1.0pt]
    \multirow{2}{*}{Models} & \multicolumn{2}{c|}{$I/I$} & \multicolumn{2}{c|}{$I/\text{SO(3)}$} & \multicolumn{2}{c}{$\text{SO(3)}/\text{SO(3)}$}\\
    & mIoU & mcIoU & mIoU & mcIoU & mIoU & mcIoU \\
    \hline
    \hline
    &&&&&&\\[-0.9em]
    PointNet++~\cite{qi2017pointnet++} & \textbf{85.56} & 82.24 & 28.97 & 32.95 & 82.10 & 78.02 \\
    DGCNN~\cite{wang2019dynamic}       & 85.13 & \textbf{85.41} & 31.48 & 32.19 & 77.60 & 78.24 \\
    &&&&&&\\[-1.1em]
    \hline
    &&&&&&\\[-1.0em]
    VN~\cite{deng2021vector}           & 81.74 & 82.06 & 81.73 & \textbf{82.05} & 81.76 & \textbf{82.10} \\
    LC~\cite{kaba2023equivariance}     & 81.04 & 81.43 & 81.07 & 81.45 & 81.05 & 81.53 \\
    PRLC~\cite{mondal2023equivariant}  & 79.38 & 79.90 & 79.39 & 79.91 & -     & - \\
    CSEConv~\cite{kim2024continuous}  & 35.80 & 32.23 & 35.80 & 32.23 & 35.51 & 33.20 \\
    ECKConv   & 83.78 & 81.19 & \textbf{83.78} & 81.19 & \textbf{83.68} & 80.99 \\
    \bottomrule[1.0pt]
    \end{tabular}}}
\end{table}

\begin{figure}[t]
    \centering
    \includegraphics[width=0.92\columnwidth, trim={1.0cm, 0.3cm, 0.3cm, 0.3cm}, clip]{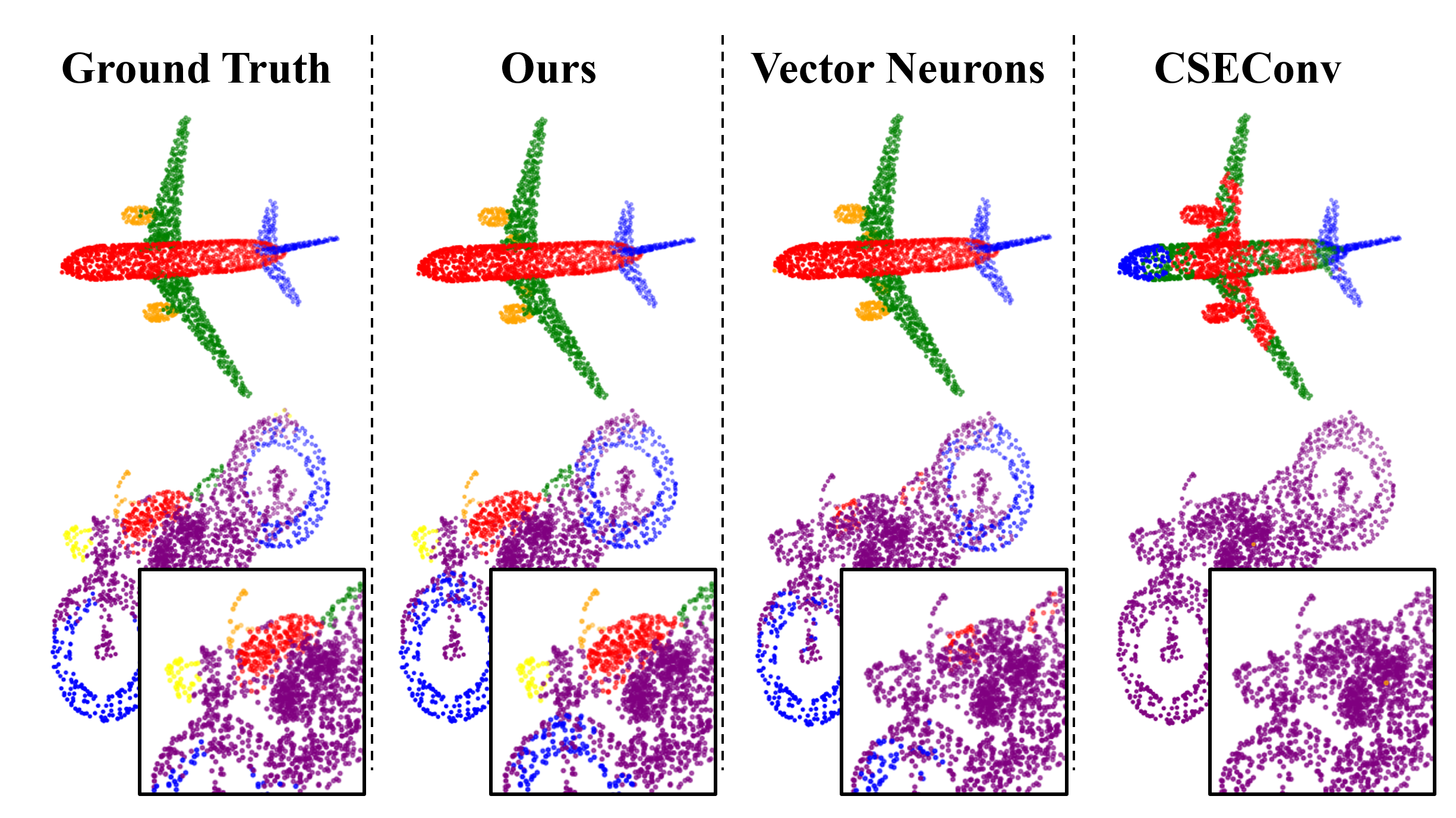}
    \caption{Visualization of the part segmentation in the ShapeNet by ECKConv (Ours), Vector Neurons~\cite{deng2021vector}, and CSEConv~\cite{kim2024continuous}.}
    \label{fig:part_seg}
\end{figure}

\begin{figure*}[t!]
    \centering
    \includegraphics[width=0.9\textwidth]{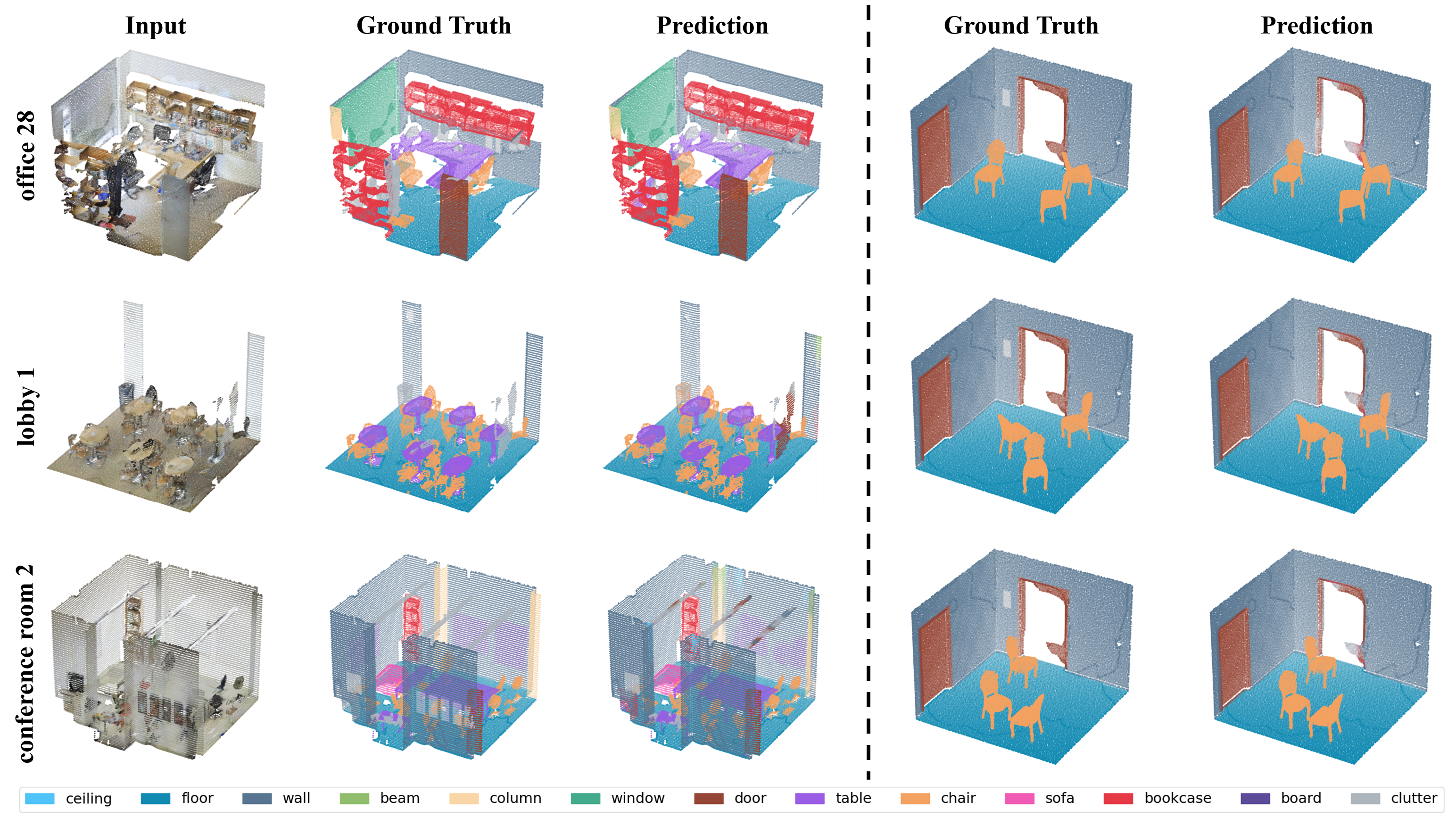}
    \caption{Semantic segmentation results by ECKConv on the Area 5 of S3DIS~\cite{armeni20163d}.
    The ceiling and wall points that obscure the view point are omitted in the figures for the better visualization.
    (Left) The input, ground truth label, and prediction results of \textit{office 28}, \textit{lobby 1}, and \textit{conference room 2} in the Area 5.
    Various indoor objects, \eg, bookcase, table, chair, and sofa, are successfully segmented from the scenes by our method.
    (Right) The simulation by randomly placing chair objects from ModelNet40~\cite{wu20153d} in the \textit{hallway 8} of Area 5. Even though the chair objects are out-of-distribution from the S3DIS, our method segments every object from the simulated scenes.}
    \label{fig:indoor_seg}
\end{figure*}
\subsection{Indoor Semantic Segmentation}\label{sec:indoor_semseg}
Our indoor segmentation model cropped 4$\;\!\!m^2$ square sub-regions and sampled the points from a single scene during its training due to the large scale of S3DIS, following previous researches~\cite{dong2022autoencoders,su2025ri}.
Similar to the part segmentation model, its U-Net architecture consisted of 30 layers of ECKConv and received point clouds with normal vectors but without colors.
It was evaluated to cover the whole points of each aligned scene in the Area 5; we divided the scene to sub-regions with one meter strides, and every point in each region was split to uniformly sampled point clouds.
We also voted 3 times to decide the final prediction as previous works~\cite{thomas2019kpconv,dong2022autoencoders,thomas2024kpconvx,su2025ri} and measured the segmentation performance with overall accuracy (OA), mean accuracy (mACC), and mean intersection over unit (mIoU).

\begin{table}[t]
    \centering
    \caption{Semantic segmentation results on the Area 5 of S3DIS. \textdagger~denotes the result reported in RI-MAE~\cite{su2025ri}.}\label{tab:indoor_seg}
    \resizebox{\columnwidth}{!}{{\tiny
    \begin{tabular}{l|ccc}
    \toprule[0.4pt]
    &&&\\[-1.5em]
    Models & OA & mAcc & mIoU \\
    &&&\\[-1.2em]
    \hline
    &&&\\[-1.13em]
    \hline
    &&&\\[-0.8em]
    DGCNN~\cite{wang2019dynamic} w/ SO(3) Aug & 74.78 & 45.06 & 35.78 \\
    KPConv~\cite{thomas2019kpconv} w/ SO(3) Aug & 84.50 & 64.66 & 57.42\\
    RIConv++~\cite{zhang2022riconv++} & 84.68 & 66.41 & 56.94 \\
    RI-MAE~\cite{su2025ri}\textsuperscript{\textdagger} & - & - & 60.3 \\
    &&&\\[-1.0em]
    ECKConv-Normal & \textbf{86.38} & \textbf{70.38} & \textbf{61.80} \\[-0.2em]
    \bottomrule[0.4pt]
    \end{tabular}}}
    \vspace{-0.5em}
\end{table}

We compared ECKConv to DGCNN and KPConv trained with SO(3) augmentation and rotation invariant baselines, \eg, RIConv++~\cite{zhang2022riconv++} and RI-MAE~\cite{su2025ri}.
\Cref{tab:indoor_seg} shows that the degradation by SO(3) augmentation was severe enough to drop the performance of non-equivariant methods lower than the rotation-invariant baselines;
it necessitates the SE(3) symmetric model scalable enough to learn large-scale problems.
Indeed, our ECKConv successfully conducted the segmentation as in the left of \Cref{fig:indoor_seg} and reached the best performance in every metric, especially 61.80 in mIoU, which is the state-of-the-art performance among rotation-invariant methods in the S3DIS segmentation to the best of our knowledge.

We also conducted a simulation where chair objects from ModelNet40, out-of-distribution data to S3DIS, were randomly placed in the \textit{hallway 8} of Area 5.
As shown in the right of \Cref{fig:indoor_seg}, our method discerned the chairs regardless of their poses on the floor.
Please refer to \Cref{fig:supp_simul} in the supplementary material for more qualitative simulation results.
They verified that not only ECKConv is scalable enough but also its locally SE(3) equivariant features plays an essential role in resolving large-scale problems.

\section{Conclusion}
We propose ECKConv, SE(3) equivariant convolutional networks for point clouds using the intertwiner framework.
We formalize the kernel domain as the double coset space of SE(3) acted by SO(2) subgroups and acquire the parameters defining SE(3)-invariant and disjoint orbits, guaranteeing the equivariance and injectiveness of the kernel domain.
Utilizing the coordinate-based networks with the Gaussian embedding,
the double coset parameters from the normalized receptive field in a ball shape of ECKConv are mapped to the weights for the basis kernel maps,
enhancing the capability to learn local geometries.
We also reformulate the explicit kernel computation and grant the memory scalability to ECKConv.
Extensive experiments in the ModelNet40, ShapeNet, and S3DIS empirically validated the enhanced performance and efficiency of ECKConv, which achieved the best performance among the state-of-the-art equivariant or rotation-invariant point cloud methods in diverse synthetic and real-world problems.
\newpage
\paragraph{Acknowledgement}
This work was partly supported by grants from the IITP (RS-2021-II211343-GSAI/10\%, RS-2022-II220951-LBA/15\%, RS-2022-II220953-PICA/15\%), the NRF (RS-2024-00353991-SPARC/15\%, RS-2023-00274280-HEI/15\%), the KEIT (RS-2025-25453780/15\%), and the KIAT (RS-2025-25460896/15\%), funded by the Korean government.
{
    \small
    \bibliographystyle{ieeenat_fullname}
    \bibliography{main}
}

\clearpage
\setcounter{page}{1}
\maketitlesupplementary

\begin{figure*}[t]
    \centering
    \includegraphics[width=\textwidth]{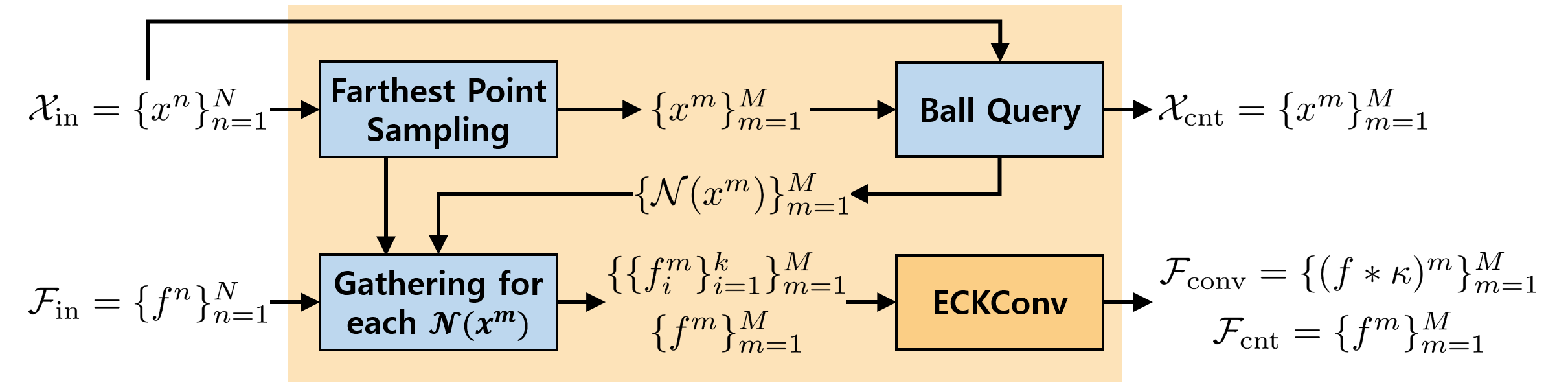}
    \caption{Detailed architectures of ECKConv block. We substitute abstract variables from \Cref{fig:layer_arch} with actual computational variables.}
    \label{fig:supp_layer_detail}
\end{figure*}

\begin{algorithm}[t]
\caption{The coset augmentation from the coordinates}\label{alg:supp_coset_gen}
\begin{algorithmic}[1]
\Require coordinates set $\mathbf{X} \in \mathbb{R}^{N \times 3}$, neighborhood size $K$
\Ensure augmented cosets $\mathbf{N} \in \mathbb{R}^{N \times 3},\;\;\;\text{s.t.}\;\;\mathbf{N} \subset S^2$
\State $\big\{\mathbf{x}^{n},\;\{\mathbf{x}^{n}_{i}\}^{K}_{i=1}\big\}^{N}_{n=1}\;\!\gets\;\!\text{K-NN}(\mathbf{X},\mathbf{X};K)$
\For{$n=1\;\text{to}\;N$}
    \State $\{\mathbf{n}^{n}_{1},\cdots,\mathbf{n}^{n}_{K}\}\;\gets\;\{\mathbf{x}^{n}_{1}-\mathbf{x}^{n},\cdots,\mathbf{x}^{n}_{K}-\mathbf{x}^{n}\}$
    \State $\mathbf{n}^{n}\;\gets\;\frac{1}{K}\sum_{i}\mathbf{n}^{n}_{i}$
    \If{$\;\;\|\mathbf{n}^{n}\|_{2}\;>\;1\mathrm{e}-5\;\;$}
        \State $\hat{\mathbf{n}}^{n}\;\gets\;\dfrac{\mathbf{n}^{n}}{\|\mathbf{n}^{n}\|_{2}}$
    \Else
        \State $\hat{\mathbf{n}}^{n}\,\gets$ the least eigenvector of $\text{PCA}(\{\mathbf{n}^{n}_{i}\}^{K}_{i=1})$
    \EndIf
\EndFor
\State $\mathbf{N} \gets \big[\hat{\mathbf{n}}^{1},\cdots,\hat{\mathbf{n}}^{N}\big]^{\top}$\\
\Return $\mathbf{N}$
\end{algorithmic}
\end{algorithm}

\section{Model Architectures and Training Setups}
We explain the details of ECKConv implementation and enumerate the training configurations of our models.
Every ECKConv model is conducted on a single RTX 3090 GPU.
Their architectures are illustrated in \Cref{fig:supp_cls_model,fig:supp_pose_reg_model,fig:supp_part_seg_model,fig:supp_semseg}.
When applying neighboring algorithms specified for point clouds, \eg, farthest point sampling, k-nearest neighbors, ball query, and PCA among point groups, we utilized the implementations from PyTorch3D~\cite{ravi2020pytorch3d}.

\subsection{Details of ECKConv Block}\label{sec:supp_layer_detail}
This section delineates the detailed computation of ECKConv block, which is also illustrated in \Cref{fig:supp_layer_detail}.
Given the input point cloud $\{(x^{n},f^{n})|x^{n}=(\mathbf{x}^{n}, \mathbf{n}^{n}), f^{n}=f(x^{n})\}^{N}_{n=1}$, the centroids for convolution $\{(x^{m}, f^{m})\}^{M}_{m=1}$, where $M \leq N$, are sampled first by farthest point sampling~\cite{qi2017pointnet++} using the coordinate $\mathbf{x}^{m}$s.
Then, a ball query groups the local neighbors with maximal $k$ neighbors $\mathcal{N}(x^{m}) = \{x^{m}_i|x^{m}_{i}=(\mathbf{x}^{m}_{i}, \mathbf{n}^{m}_{i})\}^{k}_{i=1}$ from the input points for each centroid $x^{m}$.
It gathers features in $\{f^{n}\}^{N}_{n=1}$ into $\{\{f^{m}_{i}\}^{k}_{i=1}\}^{M}_{m=1}$ by every $\mathcal{N}(x^{m})$, and the ECKConv layer operates on every neighborhood and its gathered feature, yielding the output features $\{(f \ast \kappa)^{m}\}^{M}_{m=1}$.
After the block, batch normalization and GELU activation are applied on the $\{(f \ast \kappa)^{m}\}^{M}_{m=1}$,
and the pair of $\{(x^{m}, (f \ast \kappa)^{m})\}^{M}_{m=1}$ is propagated as a new point cloud towards the next layer.
The above process is summarized in \Cref{alg:supp_conv_block}.

\subsection{Coset Augmentation of Point Clouds}\label{sec:supp_coset_algo}
If the ground truth normal vector were not given in the inputs,
we heuristically lift the 3D coordinate to the coset in $S^2$ as \Cref{alg:supp_coset_gen}.
It averages every difference between the neighbor points and the centroid and normalize the average into a unit vector unless its norm is not too small.
If it is, it exploits PCA on the neighbors to estimate the normal vector for the centroid.
We also experimented exploiting only the normal vectors estimated by PCA before the ECKConv layers,
but \Cref{alg:supp_coset_gen} empirically performed better in the ModelNet40 experiments.
We applied this to every our models except \textbf{ECKConv-Normal} in the ModelNet40 classification and the S3DIS segmentation experiments. 

\begin{algorithm}[t]
\caption{The process of ECKConv block}\label{alg:supp_conv_block}
\begin{algorithmic}[1]
\Require coordinates $\mathcal{X}_\text{in}\in\mathbb{R}^{N \times 6}$, features $\mathcal{F}_\text{in}\in\mathbb{R}^{N \times C_\text{in}}$
\Ensure $\{x^{m}\}^{M}_{m=1}$, $\{f^{n}\}^{N}_{n=1}$
\State $\{x^{n}\}^{N}_{n=1} \gets \mathcal{X}_\text{in}$; $\{f^{n}\}^{N}_{n=1} \gets \mathcal{F}_\text{in}$
\State $\{x^{m}\}^{M}_{m=1}\gets\text{farthest-point-sampling}(\{x^{n}\}^{N}_{n=1},\,M)$
\State Gather $\{f^{m}\}^{M}_{m=1}$ for $\{x^{m}\}^{M}_{m=1}$
\State $\mathcal{X}_\text{cnt} \gets \{x^{m}\}^{M}_{m=1}$, $\mathcal{F}_\text{cnt} \gets \{f^{m}\}^{M}_{m=1}$
\State $\{\mathcal{N}(x^{m})\}^{M}_{m=1} \gets \text{ball-query}(\{x^{m}\}^{M}_{m=1},\;\{x^{n}\}^{N}_{n=1},\,k)$
\State Gather $\{\{f^{m}_{i}\}^{k}_{i=1}\}^{M}_{m=1}$ for $\{\mathcal{N}(x^{m})\}^{M}_{m=1}$
\State $\{(f\ast\kappa)^{m}\}^{M}_{m=1} \gets \text{ECKConv}(\{x^{m}\}^{M}_{m=1},$\newline
\hspace*{12.6em}$\{\mathcal{N}(x^{m})\}^{M}_{m=1} ,$\newline
\hspace*{12.6em}$\{\{f^{m}_{i}\}^{k}_{i=1}\}^{M}_{m=1})$
\State $\mathcal{F}_\text{conv} \gets \{(f\ast\kappa)^{m}\}^{M}_{m=1}$\\
\Return $\mathcal{X}_\text{cnt},\,\mathcal{F}_\text{cnt},\,\mathcal{F}_\text{conv}$\newline\Comment{$\mathcal{F}_\text{cnt}$ is utilized only in the residual connection}
\end{algorithmic}
\end{algorithm}

\subsection{Point Feature Propagation Module}\label{sec:supp_pfp}
We adopt this module suggested by \citet{qi2017pointnet++} in our segmentation models,
up-sampling point features from a \textit{coarse} point cloud $\mathcal{X}' \in \mathbb{R}^{M \times 3}$ onto a \textit{fine} one $\mathcal{X} \in \mathbb{R}^{N \times 3}$ such that $N > M$.
Querying the top-K nearest points from the coarse point cloud per each fine point,
local features from queried points are gathered with normalized weights reciprocal to distances and concatenated to features of the query point:
\begin{equation}
\begin{gathered}
    f^\text{FP}_i = \frac{\sum^K_{k=1}w_{ik} f'_{ik}}{\sum^K_{k=1}w_{ik}},\;\;w_{ik} = \frac{1}{\|x_i - x'_{ik}\|^2_2},
\end{gathered}
\end{equation}
where $x_i \in \mathcal{X}$, and $x'_{ik} \in \mathcal{X}'$ is queried to $x_i$ by KNN with corresponding feature $f'_{ik}$.
Then the upsampled feature $f^\text{FP}_i$ is concatenated to $f_i$, the feature of $x_i$ before the feature propagation,
and the module propagates a new point feature set $\big\{\big[f_i, f^\text{FP}_i\big]\big\}^N_{i=1}$ to the following ECKConv layer.
We set $K=1$ and $K=3$ for object part and indoor semantic segmentation model each.

\subsection{Training Configurations}

\subsubsection{ModelNet40 Classification}
We let every baseline training follow its original configuration of training: the number of input points, whether to utilize a normal vector, augmentation except SO(3) rotation, and other optimizer hyperparameters.
For our method, we sampled 1024 points for each point cloud and applied the scaling augmentation during the training, \ie, the scales of the object in XYZ directions were randomly rescaled in range $(\frac{2}{3},1.5)$.
Each minibatch contained 16 point clouds, and its cross entropy loss was smoothed with parameter 0.2 and minimized by Adam optimizer, whose initial learning rate was $1\mathrm{e}-4$ and scheduled to $1\mathrm{e}-6$ by cosine annealing scheduler during 200 epochs.

\subsubsection{ModelNet40 Pose Registration}
Our ECKConv layers extracted local features from the two distinct inputs, named the source and target.
We adopted Deep Closest Point~\cite{wang2019deep} to determine the proper interpolation value of source coordinates that match the target coordinates.
A variant of the transformer encoder-decoder structure plays a critical role in DCP to find the interpolation score between points.
Every deep learning module here shares its parameters for both source and target.
Finally, the relative pose $\mathbf{T}_\text{pred}=(\mathbf{R}_\text{pred}, \mathbf{t}_\text{pred})$ from the source to target can be readily computed by applying the singular value decomposition and mean difference.
Following \citet{wang2019deep}, the following loss was minimized:
\begin{equation}
    \mathcal{L}(\mathbf{T}_\text{pred},\mathbf{T}_\text{gt})=\|\mathbf{R}_\text{pred}\mathbf{R}_\text{gt}^\top-\mathbf{I}\|^2_2 + \|\mathbf{t}_\text{pred}-\mathbf{t}_\text{gt}\|^2_2.
\end{equation}
The training was conducted for 50 epochs, and each minibatch contains 16 pairs without augmentation.
Adam optimizer minimizes the loss with $1\mathrm{e}-4$ initial learning rate scheduled to $1\mathrm{e}-6$ by cosine annealing scheduler.

\subsubsection{ShapeNet Part Segmentation}
We configured the number of points in a point cloud to 2048 in every model for fair comparison,
since the number of samples in a point cloud affects the IoU metrics.
It might influence the performance discrepancy of PRLC~\cite{mondal2023equivariant} between its reported result and our experiment, whose number of points is 1024 in its original implementation.
We applied an identical augmentation protocol from our classification model training.
A cross-entropy loss was computed on every point of each input and smoothed with parameter 0.2.
The loss on minibatch containing 16 point clouds was minimized by Adam optimizer with hyperparameter $\beta_1=0.5$, and its initial learning rate was scheduled from $1\mathrm{e}-4$ to $1\mathrm{e}-6$ by the cosine annealing during 250 epochs.

\subsubsection{S3DIS Semantic Segmentation}
As stated in \Cref{sec:indoor_semseg}, we adopted the training protocols by \citet{zhang2022riconv++,dong2022autoencoders,su2025ri} to crop a $4\,\!m^2$ region and sample 4096 points within it.
A single room in the area was resampled multiple times in a single epoch to cover its whole scene with the cropped regions.
The cropped regions were randomly scaled in range $(\frac{2}{3},1.5)$ along the XY directions and flipped along the X-axis with probability 0.5 during the training.
Besides, the label weight for cross-entropy loss was precomputed since the imbalance among the semantic category is severe in the S3DIS.
The cross entropy loss was minimized by Adam optimizer with learning rate $1\mathrm{e}-4$ during 80 epochs.

\section{Setups for CSEConv Baselines}
Since CSEConv~\cite{kim2024continuous} is our main baseline with respect to the scalability,
we detail the setups for CSEConv baselines in the pose registration and part segmentation, which were not experimented in its original paper.
Every CSEConv baseline followed the initial feature generation protocol suggested by \citet{kim2024continuous}.

\subsection{Pose Registration}
As stated in the main paper, it shared the identical model and training configurations to the ECKConv model, e.g., the basis dimension of coordinate-based networks (though CSEConv uses random Fourier feature~\cite{rahimi2007random}), the number of layers, the dimension of each hidden layer, the cardinality of neighbors per layer, and training hyperparameters.
However, 6 GPUs were required to train this model in our workstation.
It shifted both source and target point clouds to the origin by subtracting geometric means since CSEConv only preserves SO(3) symmetry.
After the extraction, the model replaced down-sampled point clouds to the original center coordinates before DCP method.

\subsection{Part Segmentation}
Our CSEConv baseline also followed U-Net style architecture but had only three down-sampling and up-sampling layers each.
It shared the identical training hyperparameters to the ECKConv experiment except for the batch size and training epochs,
which reduced to 6 point clouds and 100 epochs.
This configuration also utilized 6 GPUs for the training in our workstation.

\subsection{Scalability Comparisons}
\begin{table}[t]
    \centering
    \caption{Memory and time costs during the training, compared between ECKConv and CSEConv baselines from the pose registration and part segmentation tasks.}
    \resizebox{0.85\columnwidth}{!}{{\small
    \begin{tabular}{c|c|cc|cc}
        \toprule[0.9pt]
        & Models & Memory (GB) $\downarrow$ & \#Batch/sec $\uparrow$ & batch size & \#params\\
        \hline
        \multirow{4}{*}{\STAB{\rotatebox[origin=l]{90}{{\fontfamily{ptm}\selectfont pose reg}}}} & \multirow{2}{*}{CSEConv} & \multirow{2}{*}{39.09} & \multirow{2}{*}{3.37} & \multirow{2}{*}{16} & \multirow{2}{*}{10.8M} \\
        &&&&& \\[-0.5em]
        & \multirow{2}{*}{ECKConv} & \multirow{2}{*}{5.37} & \multirow{2}{*}{7.39} & \multirow{2}{*}{16} & \multirow{2}{*}{4.1M} \\
        &&&&& \\[-0.1em]
        \hline
        \hline
        \multirow{4}{*}{\STAB{\rotatebox[origin=l]{90}{{\fontfamily{ptm}\selectfont part seg}}}} & \multirow{2}{*}{CSEConv} & \multirow{2}{*}{66.46} & \multirow{2}{*}{2.04} & \multirow{2}{*}{6} & \multirow{2}{*}{14.9M} \\
        &&&&& \\[-0.5em]
        & \multirow{2}{*}{ECKConv} & \multirow{2}{*}{10.26} & \multirow{2}{*}{3.80} & \multirow{2}{*}{16} & \multirow{2}{*}{22.5M} \\[0.8em]
        \bottomrule[0.9pt]
    \end{tabular}}}
    \label{tab:supp_cost}
\end{table}
We additionally measure the cost of CSEConv baselines during the training and compare them with our models in \Cref{tab:supp_cost}.
Since the training setups which affected their efficiency varied among experiments,
their batch and parameter sizes are reported as well.

We observed the consistency in their training efficiency,
which matches the scalability statement in Proposition \ref{propo}.
When models shared the identical architecture as the pose registration task,
ECKConv was more efficient than CSEConv both in memory and time costs.
Our method also consumed fewer resources than CSEConv while processing more parameters and inputs in the part segmentation experiment.
These results supplement the scalability of ECKConv,
which effectively exploits computation resources and is scalable to obtain sufficient model capacity.
\newpage
\begin{figure*}[t]
    \centering
    \begin{subfigure}[b]{\textwidth}
    \centering
    \includegraphics[width=\textwidth]{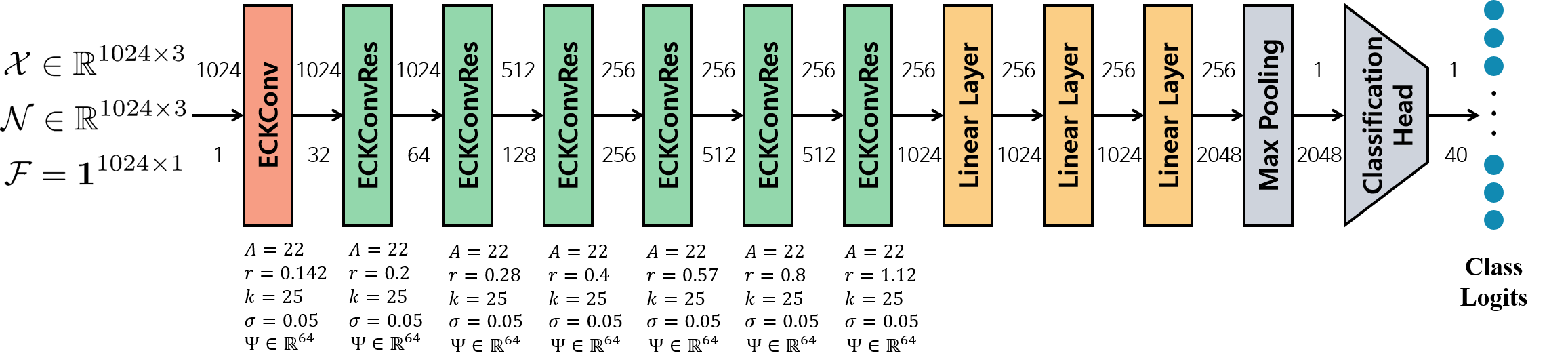}
    \caption{\label{fig:supp_cls_model}Classification model.}\vspace{0.6cm}
    \end{subfigure}
    \centering
    \begin{subfigure}[b]{\textwidth}
    \centering
    \includegraphics[width=\textwidth]{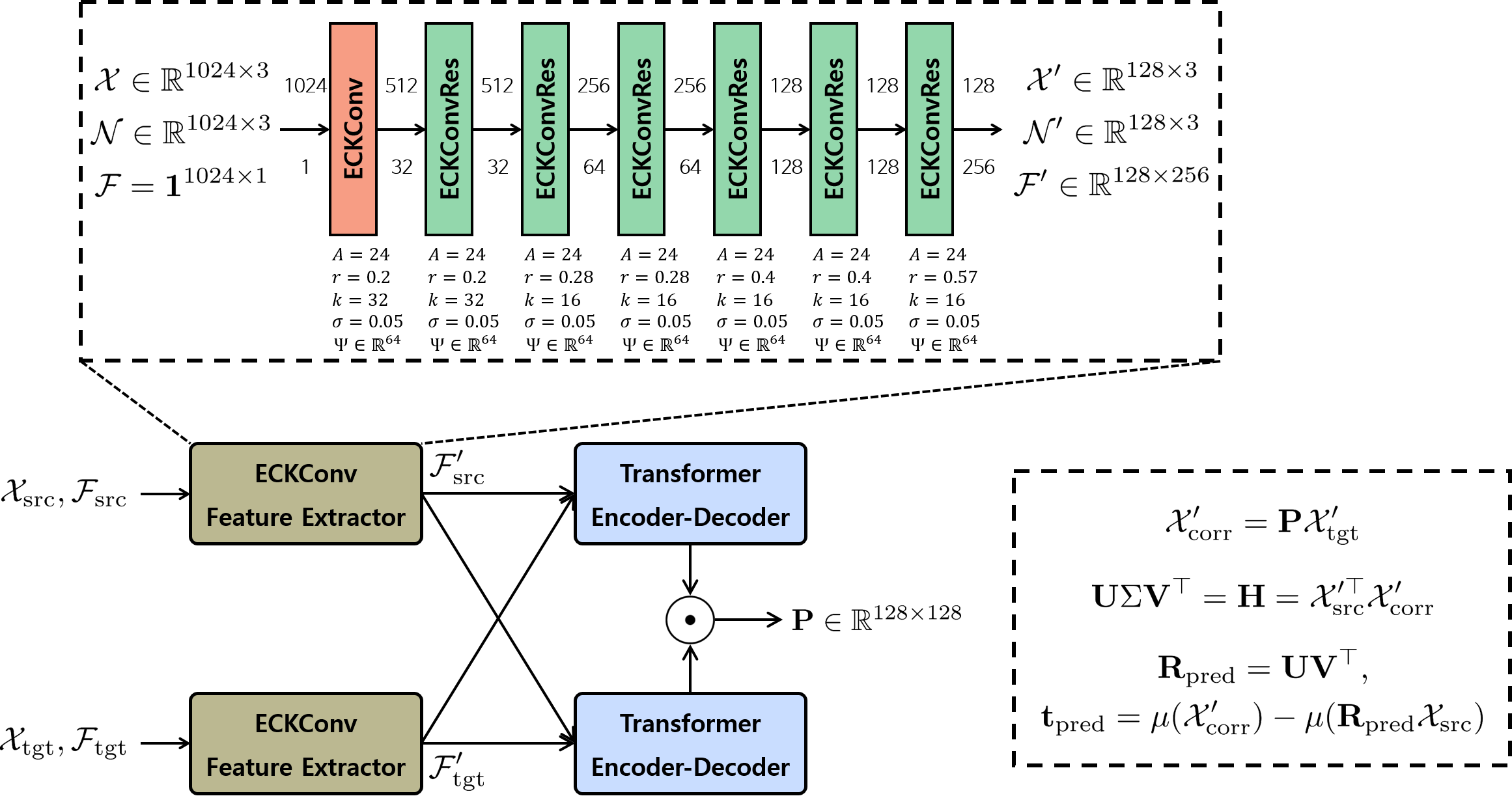}
    \caption{\label{fig:supp_pose_reg_model} Pose registration model.}
    \end{subfigure}
    \caption{Architectures of ECKCov models for ModelNet40~\cite{wu20153d} experiments. Since point clouds only contain 3D coordinates, we initialize its point feature as one vector per every point. An arrow between ECKConv layers designates the number of points and the dimension of each feature on its above and below. We also specify ECKConv hyperparameters below, such as the number of anchor points ($A$), a ball query radius ($r$), the maximal number of sampled points ($k$), the standard deviation of Gaussian kernel ($\sigma$), and the basis dimension of Gaussian embedding ($\Psi$)~\cite{zheng2021rethinking,zheng2022trading}. Linear layers map every point feature with weight-sharing networks, and we omit batch normalization and activation next to these layers in the above figures.}\label{fig:supp_modelnet_models}
\end{figure*}

\newpage

\begin{figure*}[t]
    \centering
    \begin{subfigure}[b]{0.85\textwidth}
    \centering
    \includegraphics[width=\textwidth]{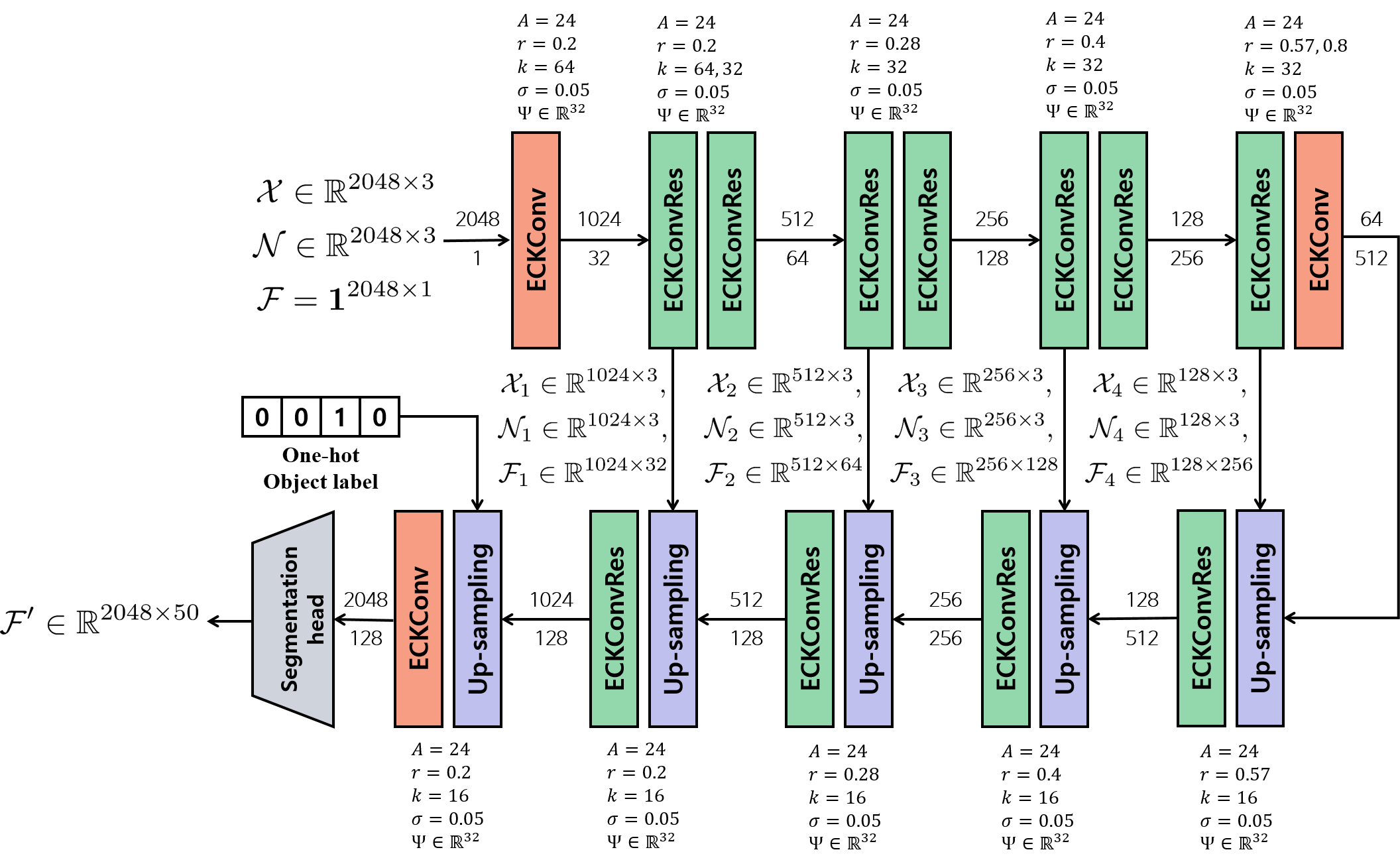}
    \caption{\label{fig:supp_part_seg_model} Part segmentation model for ShapeNet~\cite{chang2015shapenet}.}\vspace{0.4cm}
    \end{subfigure}
    \centering
    \begin{subfigure}[b]{\textwidth}
    \centering         
    \includegraphics[width=\textwidth]{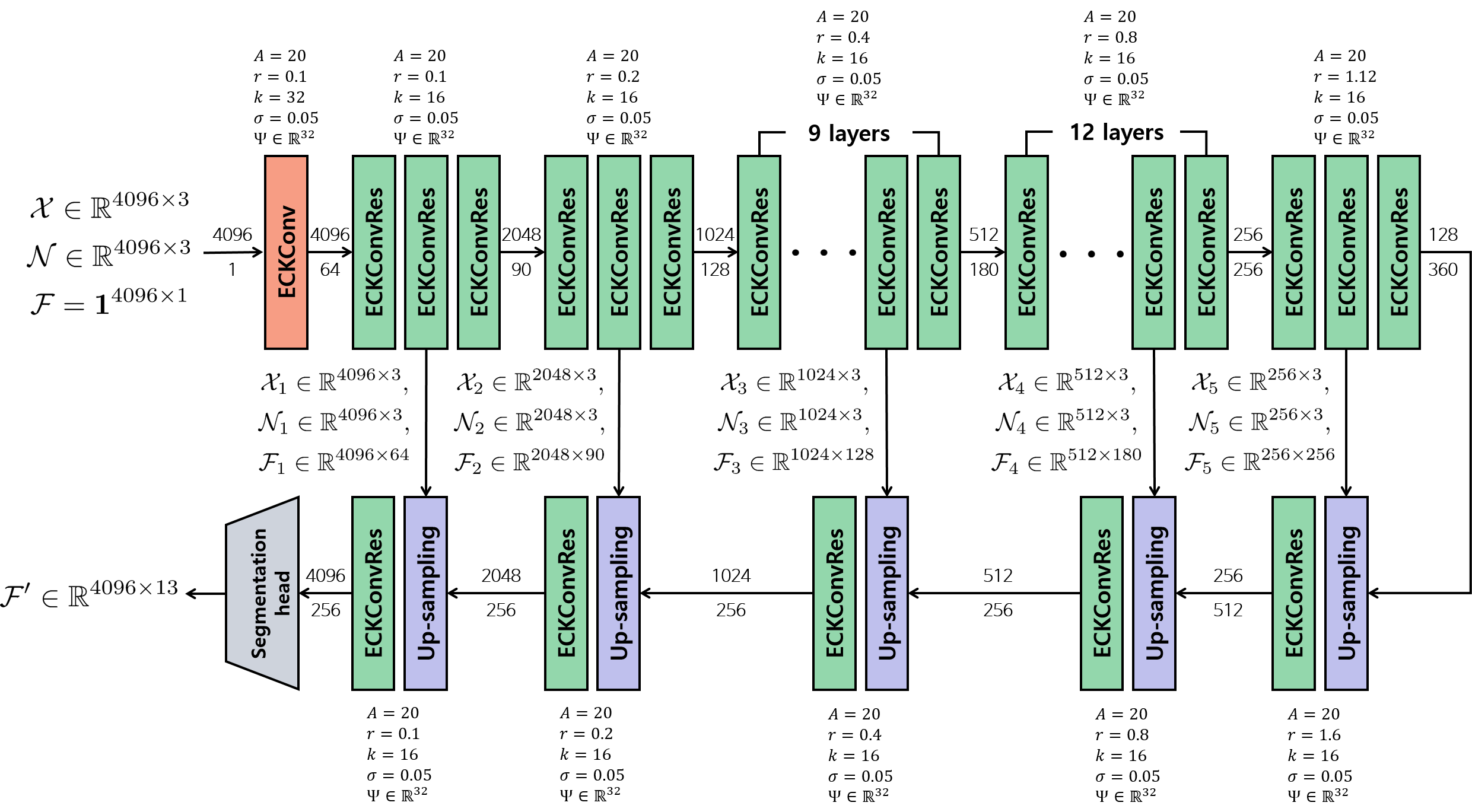}
    \caption{\label{fig:supp_semseg} Semantic segmentation model for S3DIS~\cite{armeni20163d}.}
    \end{subfigure}
    \caption{Architectures of ECKConv models fors segmentation tasks in the ShapeNet and S3DIS. We applied U-Net~\cite{ronneberger2015u} style architecture and up-sampled point features using Point Feature Propagation~\cite{qi2017pointnet++} module, which assigns coarse point features to its top-K closest neighbors skipped from previous layers, weighted by the reciprocal of the point distance in the coordinate space.}\label{fig:supp_applications}
\end{figure*}

\begin{figure*}[t]
    \centering
    \includegraphics[trim={0.5cm 0.5cm 0.5cm 0.5cm}, clip, width=0.9\textwidth]{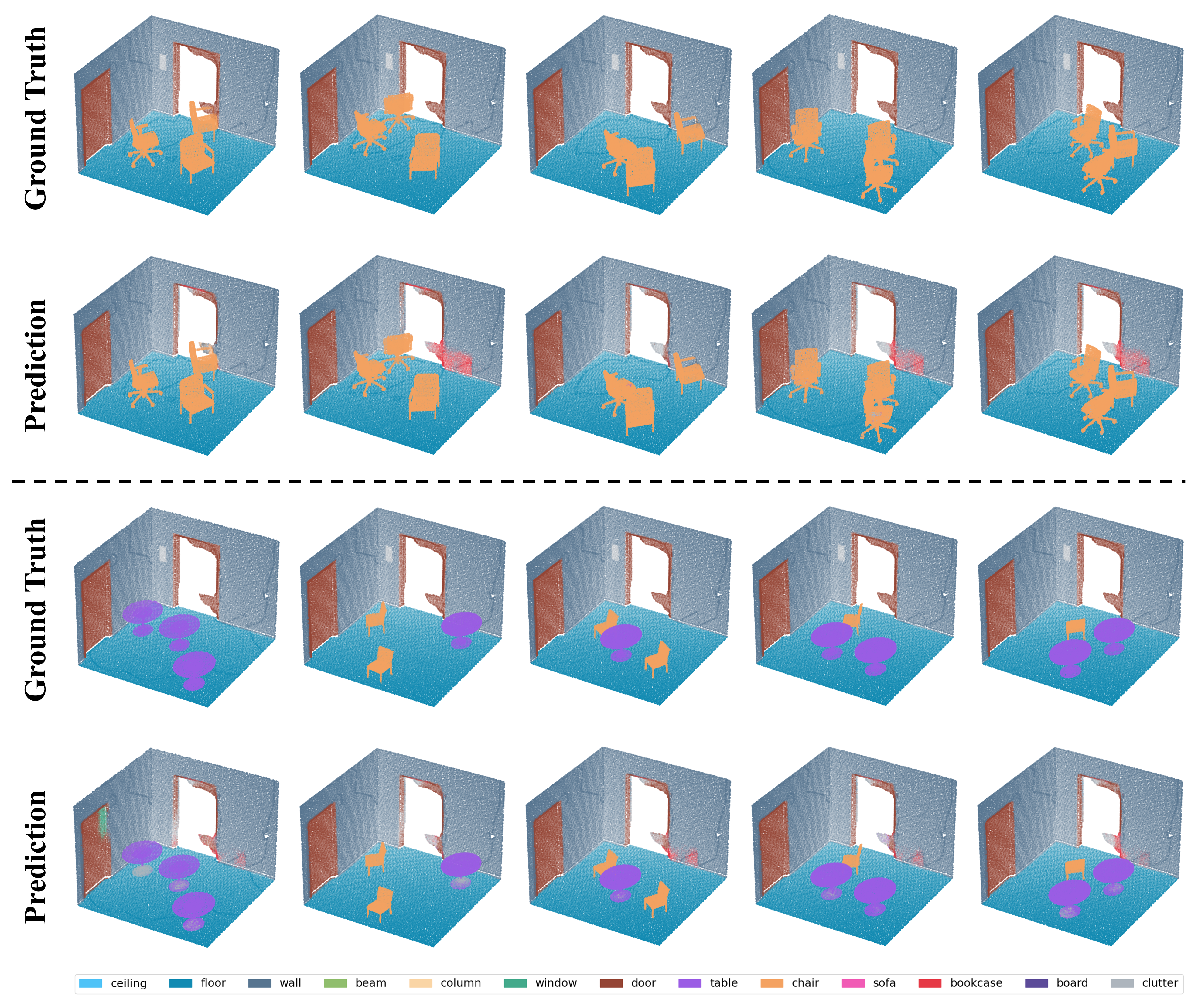}
    \caption{Qualitative results in S3DIS simulation experiments. (top) We diversified the objects from the chair category of ModelNet40. Our model consistently succeeded to discern chair objects with different shapes and poses.
    (bottom) We diversified the object category, adding the table object from ModelNet40. Our model showed consistent results on segmenting the table objects regardless of their numbers or poses, although noisy results on the bottom part of tables were also observed.}\label{fig:supp_simul}
\end{figure*}

\section{Ablations on Hyperparameter Sensitivity}
\begin{table}[t]
    \centering
    \caption{Ablation experiments on the hyperparamters of ECKConv in the ModelNet40 classification task, where models received randomly rotated inputs only during the evaluation ($I/\text{SO(3)}$). Default values are $A=22$, $\Psi\in\mathbb{R}^{64}$, and $K=32$.}\label{rebuttal:tab}
    \resizebox{\columnwidth}{!}{
    \begin{tabular}{c|cc||cc||cc}
    \toprule[0.9pt]
        \multirow{2}{*}{\STAB{Ablation\\Studies}} & \multicolumn{2}{c||}{\#Anchor Bases} & \multicolumn{2}{c||}{$\text{Gau}(\cdot)$ Bandwidth} & \multicolumn{2}{c}{\#Neighbors for Alg 1.} \\
                                    & $A=1$ & $A=64$ & $\Psi \in \mathbb{R}^{4}$ & $\Psi \in \mathbb{R}^{256}$ & $K=8$ & $K=128$ \\
                                    \hline
        Acc (\%)                    & 89.22 & 90.19 & 88.86 & 90.32 & 90.32 & 89.95 \\
        Mem (GB) $\downarrow$       & 4.36  & 6.57  & 1.90  & 14.51 & 4.90  & 4.90 \\
        $\#\text{Batch}/\text{sec}$ $\uparrow$ & 10.20 & 5.62  & 10.65 & 2.69  & 7.58  & 6.83 \\
    \bottomrule[0.9pt]
    \end{tabular}}
\end{table}
As the kernel structure and the coordinate lifting to coset space of ECKConv have diverse hyperparameters,
we conducted ablation experiments in the ModelNet40 classification task to verify the sensitivity of ECKConv against such hyperparameters.
The experiments varied the number of anchor bases ($A$), the bandwidth of Gaussian embedding ($\Psi$), and the number of neighbors during the coset lifting in \Cref{alg:supp_coset_gen} ($K$) to extreme values.
\Cref{rebuttal:tab} demonstrates that the efficiencies of the variants were affected significantly when the hyperparameters related to the kernel ($A,\,\Psi$) varied,
but their performances remained relatively robust.
We conjecture that linear layers from skip connections in our method help maintain their performances.

\section{Discussion on the Limitations}
Although ECKConv proposes an SE(3)-symmetric and scalable convolution applicable to large-scale point cloud problems,
its kernel is isotropic to SE(3) actions due to the utilization of the double coset element,
which is composed of coset elements grouped by left actions of SO(2),
and the confinement to scalar-type features.
These factors decrease the directional expressivity of ECKConv to certain types of point cloud tasks where the point feature should be anisotropic, \eg, normal estimation, molecule structure prediction, and n-body problems.
Expanding the resolvable domain of ECKConv would be an important future research direction that balances our scalable architecture with the inefficient yet expressive kernel designs from the previous steerable convolution studies.

\end{document}